%% file: article_v3.tex
\documentclass[hidelinks,onefignum,onetabnum]{siamart251216}

\input{shared}

\ifpdf
\hypersetup{
  pdftitle={Symmetric NMF},
  pdfauthor={R. Swart, and J. J. Brust}
}
\fi

\begin{document}

\maketitle

% REQUIRED
% with nonnegative rectangular $W$
\begin{abstract}
Symmetric nonnegative matrix factorization (Symmetric NMF) approximates a matrix as $WW^T$ with nonnegative rectangular factor $W$. It has broad applications in graph clustering and machine learning.
In contrast to the NMF, projected gradient methods for the symmetric problem had been associated with slow convergence.
%historically been dismissed due to slow convergence.
To address this, we introduce SNMPBB, the first adaptation of nonmonotone projected Barzilai-Borwein methods to Symmetric NMF, demonstrating that gradient algorithms are significantly more effective than previously understood.
We further extend SNMPBB to graph clustering using the graph Laplacian regularization (Graph-SNMPBB) and to large problems with low-rank approximations (LAI-SNMPBB).
For all variants we prove global convergence to first-order stationary points and also that Barzilai-Borwein curvature information is preserved with randomized approximations.
On synthetic data, SNMPBB achieves 6 times speedup over the alternative SymANLS for similar residuals, with advantages growing at higher ranks.
Across six real-world clustering benchmarks, Graph-SNMPBB matches or exceeds SymANLS accuracy.
Lastly, LAI-SNMPBB outperforms state-of-the-art LAI-SymPGNCG on 34 SuiteSparse matrices in both runtime and residual quality.
\end{abstract}

% REQUIRED

% REQUIRED

\section{Background}
%\subsection{NMF}
% \edits{and machine learning that attempts to factor a nonnegative matrix into two with reduced rank}
Nonnegative matrix factorization (NMF) is a widely used tool for sparse data extraction and machine learning.
The nonnegativity constraint naturally aligns with many real-world datasets and allows for direct interpretations (Guo et al. \cite{GuLiLi24}).
Specifically, the NMF represents each data column as a positive combination of nonnegative ``basis'' vectors, allowing the data to be decomposed into a sum of topics or features.
% Each column of the data can be represented as a positive combination of nonnegative ``basis'' vectors.
% Each basis vector may be interpreted as a feature or topic and each coefficient vector can be interpreted as how much each basis vector contributes to a given data column.
Applications of NMF include dimensionality reduction in images (Li et al. \cite{8809826}, and Jing et al. \cite{JiZhNg12}), text documents (Shahnaz et al. \cite{SHAHNAZ2006373}, and Tu et al. \cite{TuEtAl18}), noise (Jaiswal et al. \cite{5946386}), bioinformatics (Pascual-Montano et al. \cite{PascualMontano2006}), and data compression (Kong et al. \cite{KongEtAl17}).
% The method is often effective at uncovering latent structures within data that are not directly observable but still have interpretable meaning using embeddings.

Mathematically, NMF factorizes $V \in \mathbb{R}_{+}^{m\times n} $ into two low rank nonnegative matrices $W \in \mathbb{R}_{+}^{m\times r} $ and $H \in \mathbb{R}_{+}^{r \times n} $, so that $V \approx WH$.
%The following objective function represents this goal:
We define the factorization by %minimizing
\begin{equation}
    \label{nmf_original_objective}
    %f_{\textnormal{nmf}}(W,H) = \left\|V-WH\right\|^2
    \underset{ W, H \ge 0 }{\textnormal{ minimize }}  f(W,H) = \frac{1}{2}\left\|V-WH\right\|^2 
\end{equation}
% where $V \in \mathbb{R}_{+}^{m\times n}$, $W \in \mathbb{R}_{+}^{m\times r},$ and $H \in \mathbb{R}_{+}^{r \times n}$.
Typically $r \ll \min(m,n)$ and the norm in \eqref{nmf_original_objective} is the Frobenius norm.
The problem itself is NP-hard (cf. Vavasis \cite{vavasis2007complexitynonnegativematrixfactorization}), because $f$ is nonconvex in $W$ and $H$. % jointly 
Furthermore, the factorization is not unique; for any $W$, $H$ pair and generalized permutation matrix $D \in \mathbb{R}^{r \times r}$, $WD$ and $D^{-1}H$ are equally valid factors  since $WH = WDD^{-1}H$.
However, when one factor is fixed (i.e., held constant), the problem becomes convex in the other.
Thus, many algorithms use a sequence of alternating subproblems, in which $W^{(i+1)}$ is solved for with fixed $H^{(i)}$ and then fixing this newly-obtained $W^{(i+1)}$ to solve for $H^{(i+1)}$.
% Thus, many algorithms use a sequence of alternating subproblems, in which $W$ is solved for with fixed $H_k$ before fixing $W_k$ to this now-solved value to find $H$.
In fact, this strategy has been widely used in Lee and Seung's multiplicative update \cite{Lee1999} and the ANLS method (Kim and Park \cite{KiPa08}).
%\edits{}
% \jbnote[layout=inline]{Here we need to say something so that readers and editors care in 2026. Perhaps, because randomization has established itself as 
% an enabling technology in machine learning and AI, recent work, Hayashi et al.(2024), exploits randomized sketching to develop new nonnegative matrix factorization methods. Relatedly, the effectiveness of gradient type algorithms in large-scale machine learning and AI drives the development of new algorithms in nonnegative matrix factorizations 
% (citations from Sec. 1.3 and perhaps additional ones). }
%\jbe{Here we need to say things so that readers, editors care now}
%\\

% The NMF continues to attract significant attention because the nonnegativity constraints allow for direct interpretations \cite{GUO2024102379}.
% Each column of $V$ can be represented as a linear combination of a ``basis'' vector in $W$ using a coefficient column in $H$.
% Each basis vector may be interpreted as a feature or topic and each coefficient vector can be interpreted as how much each basis vector contributes to a given column in $V$.
% Applications of NMF include topic identification in text documents \cite{TU2018203}, image processing \cite{6226461}, bioinformatics \cite{PascualMontano2006}, linear clustering, and data compression\cite{info8010026}.
% The method is often effective at uncovering latent structures within data that are not directly observable but still have interpretable meaning using embeddings. %\\

\subsection{Notation}
\label{sec:notation}
We use Householder notation so that a matrix, vector and scalar are denoted as 
$A,a,\alpha$. Unless otherwise specified, the norm $ \| \cdot \| $ represents the Frobenius norm,
$ \| \cdot \| = \| \cdot \|_F $, or the corresponding vector 2-norm. The absolute value of a matrix $ | A | $ represents elementwise absolute values $ | a_{ij} |  $ .   The Frobenius inner product
of two $ m \times n $ matrices $A$ and $B$ is $ \langle A, B \rangle = \sum_{i,j} A_{ij} B_{ij} $,
and therefore $ \| A \|^2 = \langle A, A \rangle $. The gradient of a matrix valued
scalar function $ f: \mathbb{R}^{m \times n} \to \mathbb{R} $ is the matrix 
$ \nabla f \in \mathbb{R}^{m \times n} $. For a function with two variables $f(A,B)$, when one variable is fixed
we use the notation $ f(A,B^{(i)}) := f(A)  $ and $ f(A^{(i)},B) := f(B)  $. The gradients are then denoted
as $ \nabla f(A) $ and $ \nabla f(B) $. The identity matrix is $I$ with the dimension depending on the context. The $i^{\textnormal{th}}$
column of the identity is $e_i$, and  $ e $ represents the vector of all ones. The projection of $A$'s elements onto the nonnegative constraint is
$ \left( \textnormal{Proj}[A] \right)_{ij} = \{ A_{ij} \: \textnormal{if} \: A_{ij} \ge 0; \: \: 0 \: \textnormal{if} \: A_{ij} < 0  \} $. We use a subscript $k$ to denote an inner iteration and superscript 
$ ^{(i)} $ for an outer iteration, e.g., $ A^{(i)}_k  $. Depending on the context we suppress the outer iterations.

\subsection{Symmetric NMF}
Note that when $ V\in \mathbb{R}^{n \times n}$ is symmetric the 
NMF in \eqref{nmf_original_objective} may yield factors $ W $ and $H$ that accurately fit the data, but are otherwise unrelated. Symmetric NMF is a variant of the nonnegative matrix factorization, and requires each factor to equal the other's transpose: $V \approx WW^T$.
$V$ may be the similarity matrix of a Graph and represents pairwise relationships (Kuang et al.~\cite{Kuang2015SymNMF}).
Symmetric NMF has been shown to be equivalent to $K$-means clustering and can cluster nonlinear graph data (see Ding et al.~\cite{dingetal}).
The Symmetric NMF problem, although it concerns only one variable, is in some ways more difficult than minimizing \eqref{nmf_original_objective}.
The objective may be defined as %function is
\begin{equation}
    \label{original_symnmf_objective}
    f(W,W) = \frac{1}{2}\left\|V-WW^T\right\|^2 % \bar{f}(W) = 
\end{equation}
This is quartic in $W$ and also nonconvex (Li et al.~\cite{sym13091757}).
Similar to the unsymmetric problem, there is no unique solution. Unconstrained solvers with additional nonnegativity projections could
be applied to the problem \eqref{original_symnmf_objective} (e.g., \cite{brust2021nonlinear,brust2024trust,brust2024useful} ),
% This paragraph could be clearer
% Symmetric NMF methods perform a soft, or fuzzy, clustering, wherein each data point is given a numerical value for each cluster, not just a label.
% This is a consequence of the factorization not imposing strict orthogonality between columns \cite{Kuang2015SymNMF}.
% For any row that represents a data point, the value of each column represents whether or not the data point belongs to that cluster (the column index is the category). Each column can have a nonzero value for each row of $W$ and therefore membership is not strict for each category or cluster. 
%\\
however because of its ample use in machine learning, prominent methods for Symmetric NMF have emerged. % been developed
Kuang et al.~\cite{Kuang2015SymNMF} proposed a Symmetric Alternating Nonnegative Least Squares (SymANLS) algorithm that has been very successful with larger datasets.
In this method, the standard NMF problem is modified by adding a ``symmetric penalty”: % term and then still solving alternately for one variable:
%These each included a penalty term and formulated the problem
\begin{equation}
    \label{sym_obj_func}
    f(W,H;\lambda) = \frac{1}{2} \left\| V - WH \right\|^2 + \frac{\lambda}{2} \left\| W - H^T \right\|^2
\end{equation}
%A factor of $\frac{1}{2}$ is added for notational convenience in the gradient expressions.
The penalty term induces symmetry for sufficiently large $\lambda$
(Li et al.\cite{9606619}). % and still enables using alternating nonnegative least-squares iterations. 
% Specifically, each subproblem for $W$ is formulated as
% \begin{equation}
%     \label{eq:nnls}
%     \min_{W \geq 0} \left\| \begin{bmatrix} 
%     H^T\\ \sqrt{\lambda} I_r\end{bmatrix}W^T - \begin{bmatrix}
%         V^T\\ \sqrt{\lambda}H
%     \end{bmatrix} \right\|
% \end{equation}
% Solving for $H$ is analogous and when alternated, defines the SymANLS method.
%(see Alg. \ref{alg:symals})
%, which is implemented in algorithm \ref{sym_anls}.
Overall, SymANLS is very robust and has good convergence behavior under mild conditions.
However, it still requires nonnegligible computational expense per iteration due to a sequence of matrix least-squares solves. %\\
The authors also proposed a Newton-based algorithm \cite{Kuang2015SymNMF}, which 
%( see Alg. \ref{alg:symnewton}).
%This approach computes a second-order Taylor expansion of the objective function and solves a sequence of constrained quadratic subproblems. %  to determine descent directions
% For each row of $W$, a reduced-dimensional Newton system is constructed using the local Hessian, and the subproblem is efficiently solved while enforcing nonnegativity.
% The method carefully selects active variables (based on gradients) to reduce problem size and applies line search for globalization.
%This algorithm 
typically converges to higher accuracy but is also computationally much more expensive.
%and faster than other methods, but each iteration is very computationally expensive and it is suitable onle for small- to medium-scale problems.

% Another method is the Symmetric Hierarchical Least Squares, or SymHALS, algorithm by Li et al.~\cite{9606619}.
% The key idea behind SymHALS is to update each column of $W$ individually while holding others fixed. It uses the partial residual $R_j = V - \sum_{i \ne j} w_i w_i^T$ to update each column $ w_j \gets \max\big(0, \frac{R_j w_j}{\|w_j\|^2} \big)  $.
% Modifying each column results in only a rank-one update for $W$; given that this is done in an iterative fashion, the method is very computationally efficient.
% SymHALS is therefore particularly effective for large-scale problems.
% However, given the lack of curvature information, it may converge slower than the Newton-based methods.
%s(see Alg. \ref{alg:symhals})

Recently, Hayashi et al.\cite{HaEtAl25} proposed the Low-rank Approximate Input SymNMF, or LAI-SymNMF approach. %\cite{HaEtAl25}.
This is a class of methods that approximate the input matrix, $V$, through randomized sketching. % , to a lower rank matrix.
Using an approximate singular value decomposition, the authors drastically lowered the per-iteration cost of several NMF algorithms by calculating matrix multiplications using low rank approximations. 
%the results of the decomposition instead of larger multiplies with the original matrix $V$.
They propose LAI-SymPGNCG, which implements the new features and is based on the PGNCG algorithm (Kim and Park, \cite{KiPa08b}). %\\
%In the following we use the notation

\subsection{Projected Gradient for NMF}
%\jbnote[layout=inline]{This section looks overall fine, just see the comment at the bottom. Try to reformulate this.}
% For the NMF (without symmetry) another class of methods is based on gradients.
% Added in here
For the NMF (without symmetry), arguably the best known method is the multiplicative update due to Lee and Seung, \cite{Lee1999}.
Related to this are projected gradient methods which use projections to enforce feasibility.
Instead of solving a sequence of nonnegative least-squares problems based on \eqref{nmf_original_objective}, these methods use projected gradients (PG). % as in \eqref{eq:nnls} 
For the problem $\min_{W \geq 0} f(W,H^{(i)})$, projected gradient descent uses the iteration:
\begin{equation}
    W_{k+1} = \text{Proj}[W_k - \alpha_k \nabla f(W_k)].
    \label{eq:original_pgd}
\end{equation}
%where $\text{Proj}[\cdot]$ denotes projection onto the nonnegative orthant $R^{m \times r}_+$.
The projection is computationally extremely inexpensive and requires only  $mr$ sign checks.
% Similar as alternating least-squares, projected gradient 
% approaches may also exploit the problem structure from \eqref{nmf_original_objective} in order to alternate gradient updates (Lin \cite{Li07}). % \eqref{sym_obj_func}
The computational efficiency of PG usually depends on appropriate step-sizes, because gradient type algorithms are sensitive to scaling and may converge slowly on problems with large curvature (Bertsekas \cite{Bertsekas99}). Since evaluating $f(W,H)$ in \eqref{nmf_original_objective} is often expensive, conventional line-search techniques that require multiple objective evaluations may be impractical.
On the other hand, simple step-size rules such as constant or exponentially
decaying sizes, e.g., $ \alpha_k = \frac{1}{k}, \: k>0  $, may not work well for highly nonlinear problems. A scaling rule that has proven effective
is the Barzilai and Borwein (BB) step-size (Barzilai and Borwein \cite{BaBo88}, \cite{BrBuErMa19}, \cite{brust2021compact}). % also Brust et al.
%
% For the vector problem $ \textnormal{min}_{W \in \mathbb{R}^{m r}} f(W)$ one of two BB steps is
% \begin{equation}
%     \alpha_k = \frac{\langle S_{k-1}, Y_{k-1}\rangle}{\langle S_{k-1}, S_{k-1} \rangle}
%     \label{eq:bb}
% \end{equation}
% where $S_k = W_k - W_{k-1}$ and $Y_k = \nabla f(W_k) - \nabla f(W_{k-1})$ \cite{BaBo88}.
% A simple update method using this BB step size is (\ref{eq:original_pgd})
% % \begin{equation}
% %     x_{k+1} = \textnormal{Proj}[x_k - \alpha_k \nabla f(x_k)]
% % \end{equation}
% (see e.g., Han et al.~\cite{HaNePr09}). %\\
%
For the NMF a method using this BB step size is the Quadratic Regularization Projected Barzilai–Borwein (QRPBB) algorithm outlined in Huang et al. \cite{Huang2015Quadratic}.
This method incorporates the above update rule after finding $Z_k$ where $Z_k$ is the solution to the quadratically regularized and strongly convex subproblem 
\begin{equation}
    \underset{ Z \ge 0 }{ \textnormal{ minimize } } f(W_k) + \langle\nabla f(W_k), Z-W_k\rangle + \frac{\rho}{2} \left\| Z - W_k \right\|^2 % omega
\end{equation}
Here $\rho>0$ is a scaling for the size of the quadratic regularization. % omega
% \begin{equation}=
%     \min_{W \in \mathbb{R}^{m \times r}_+} \phi(W_k, W) = f(W_k) + \langle\nabla f(W_k), W-W_k\rangle + \frac{L}{2} \left\| W - W_k \right\|^2
% \end{equation}
A line search is then performed on the descent direction created by taking a gradient descent step from $Z_k$ with the Barzilai-Borwein step-size, i.e.,  $D_k = \text{Proj}[Z_k - \alpha_k \nabla f(Z_k)] - Z_k$ to update $W$ as $W_{k+1} = Z_k + \beta_k D_k$
with step-size $ \beta_k > 0 $.
A modification to QRPBB is the Nonmonotone Projected Barzilai-Borwein (NMPBB) algorithm, which adds a relaxation factor $ \varphi > 0$ for $W_{k+1} = Z_k + \varphi \beta_k D_k$ as well as a nonmonotone line search (Li et al. \cite{sym16020154}).
%
% In NMPBB, sufficient decrease is determined by monitoring $F_k = f(W_k) + \eta_{k-1}(F_{k-1} -f(W_k))$ with $F_0 = f(W_0)$.
% The line-search is therefore looking for $\beta_k$ such that $F_{k+1} \leq F_k + \theta_k \beta_k \langle \nabla f(Z_k), D_k \rangle$ with parameter $\theta_k$.
%
QRPBB has been shown to converge faster than established methods like NeNMF (Guan et al. \cite{GuTaLuYu12}), PG (Lin \cite{Li07}) and APBB2 (Han et al. \cite{HaNePr09}) while
NMPBB outperformed QRPBB (see the experiments in \cite{Huang2015Quadratic} and \cite{sym16020154}).
A pseudocode of NMPBB is in Algorithm \ref{nmpbb}.
Surprisingly, vanilla projected gradient for the \emph{symmetric} NMF has been shown to converge significantly slower than for the unsymmetric factorization
(Zhu et al.~\cite{Zhu2018DroppingSF}).

\subsection{Contributions}
\label{sec:contrib}
% \edits{Is it ok that we are repeating a lot of things from the previous sections verbatim and in close proximity?}
% Despite the effectiveness of nonmonotone projected gradient methods for unsymmetric NMF, these techniques have not been applied to the symmetric variant.
Projected gradient methods have traditionally been overlooked for symmetric NMF due to slow convergence and poor robustness compared to alternating least-squares approaches (Zhu et al. \cite{Zhu2018DroppingSF}, Hayashi et al. \cite{HaEtAl25}).
This work demonstrates that with proper adaptations projected gradients for symmetric NMF may outperform established alternatives in both speed and accuracy.
To this end, we introduce SNMPBB, the first adaptation of alternating nonmonotone Barzilai-Borwein methods to symmetric NMF.
Different from naive symmetric extensions that set $H = W^T$, we employ a penalty-based coupling that maintains two variables while enforcing symmetry through gradients.
For graph clustering, we extend this to Graph-SNMPBB by incorporating graph Laplacian regularization to address sometimes poor soft-clustering behavior.
Furthermore, we prove global convergence to first-order stationary points under this penalty-based formulation and establish that graph regularization preserves this convergence.
Moreover, we interface the method with the recent Low-rank input (LAI) methodology (Hayashi et al.~\cite{HaEtAl25}) and show that the BB step-size is unchanged by the low-rank approximation error.
%with the gradient bias bounded above.
For synthetic matrices, SNMPBB has a 6 times speedup over SymANLS for convergence and 2 times speedup over SymNewton while consistently achieving a better relative residual.
Across six real-world benchmarks, Graph-SNMPBB requires far less time than SymANLS to reach high clustering accuracy.
Further, LAI-SNMPBB outperforms state-of-the-art LAI-SymPGNCG on 34 SuiteSparse matrices in both speed and residual quality; we also found that early stopping of inner iterations (3-5 steps) prevents overfitting to approximation error.

Figure~\ref{fig:reuters_motivation} provides an illustrative example of our algorithm: on the MNIST dataset, Graph-SNMPBB achieves competitive clustering accuracy in substantially less time than a standard PGD algorithm.

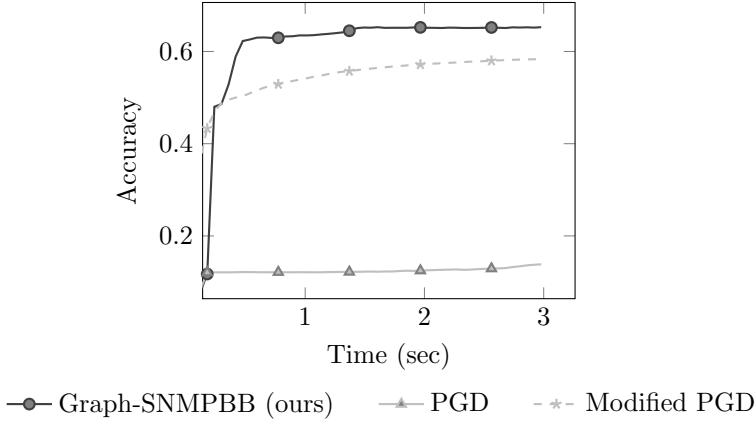
\begin{figure}
    \centering
    \hspace*{-12mm}
    \begin{tikzpicture}
\begin{axis}[
    xlabel={Time (sec)},
    ylabel={Accuracy},
    grid=none,
    width=.5\linewidth,
    height=5.5cm,
    mark repeat=10,
    xmin=0.14
]
\addplot+[mark=*, color=darkgray, thick,mark options={fill=gray, draw=darkgray}] table[x index=0, y index=1, col sep=comma] {data/2k2k.csv};
% \addplot+[mark=square*, color=gray, thick,mark options={fill=lightgray, draw=gray}] table[x index=0, y index=2, col sep=comma] {data/2k2k.csv};
\addplot+[mark=triangle*, color=lightgray, thick,mark options={fill=lightgray, draw=gray}] table[x index=0, y index=3, col sep=comma] {data/2k2k.csv};
\addplot+[mark=star, color=lightgray, thick,dashed,mark options={fill=black!20, draw=lightgray}] table[x index=0, y index=4, col sep=comma] {data/2k2k.csv};

\end{axis}
    \end{tikzpicture}
    \vspace{5pt}
\begin{tikzpicture}
\begin{axis}[
            hide axis,
            xmin=0, xmax=1,
            ymin=0, ymax=1,
            legend columns=5,
            legend style={
                at={(0.8,1.05)}, % 0.5
                anchor=south,
                draw=none,
                /tikz/every even column/.append style={column sep=0.5cm}
            },
        ]
\addlegendimage{mark=*, color=darkgray, thick, mark options={fill=gray}}
\addlegendentry{Graph-SNMPBB (ours)}
\addlegendimage{mark=triangle*, color=lightgray, thick, mark options={fill=gray}}
\addlegendentry{PGD}
\addlegendimage{mark=star, color=lightgray, thick, dashed, mark options={fill=black!20}}
\addlegendentry{Modified PGD}
\end{axis}
\end{tikzpicture}

    \caption{Clustering accuracy vs.\ time on MNIST averaged over 10 runs.
    Graph-SNMPBB reaches near-peak accuracy faster than both PGD \cite{zhang2023} and Modified PGD (using a larger step-size), motivating the approach developed in this paper.}
    \label{fig:reuters_motivation}
\end{figure}

% For the symmetric nonnegative matrix factorization, we propose a new projected gradient method. In particular, we leverage effective strategies developed for unsymmetric NMF, such as nonmonotone line search, Barzilai-Borwein step-sizes, and quadratic regularization. 
% We first modify algorithm NMPBB to solve symmetric problems using a penalty term to enforce symmetric constraints.
% To be useful for graph clustering, we further modify the method by adding a new regularization term and preprocessing of the input matrix. 
% An analysis proves global convergence of the method and a desirable curvature property for approximate inputs.    
% In a comprehensive set of experiments we show that the new method is less expensive when compared to well known algorithms for symmetric NMF, such as SymANLS and SymNewton.
% Furthermore, we interface the method with the recent LAI methodology (Hayashi et al.~\cite{HaEtAl25}) and demonstrate advantages against state-of-the-art LAI-derived procedures.

\section{Algorithm}

\subsection{SNMPBB}
A fast and effective method for the nonnegative factorization in \eqref{nmf_original_objective} is the % matrix 
nonmonotone projected Barzilai-Borwein method (NMPBB). Algorithm \ref{nmpbb} summarizes iterates for $W_{k+1}$, while $H_{k+1}$ may be computed equivalently.
The parameter $\eta_k$ accounts for the previous objective function values and $ \theta_k $ defines a sufficient decrease condition in the nonmonotone line search (this may be constant, e.g., $10^{-3}$).
% \begin{equation}
% \label{eq:func_hist}
% F_k = 
% \begin{cases}
%       f(W_0) & \text{if}\; k=0\\
%       f(W_k) + \eta_{k-1}(F_{k-1}-f(W_k)) & \text{if}\; k\geq1\\
% \end{cases}
% \end{equation}
The parameter $ \varphi $ is the ``relaxation" factor that increases the descent magnitude.
Furthermore, $\alpha_{\text{min}}$ and $\alpha_{\text{max}}$ are the fixed constants for Barzilai-Borwein calculations.

 % \begin{center}
 % \setlength{\fboxsep}{-1pt}
 % \setlength{\fboxrule}{1pt}
 % \fbox{
 
%  \begin{minipage}[t]{0.85\textwidth}
% \vspace{0.3cm}
% \raggedright \noindent\textbf{Algorithm 2.1:} Nonmonotone Projected Barzilai-Borwein (NMPBB) \\
% (Symmetric when using $f$ from \eqref{sym_obj_func} : (SNMPBB))
% \rule{10.85cm}{1pt}
\begin{algorithm}[tb]
\caption{Nonmonotone Projected Barzilai-Borwein (NMPBB) } %\\
%(Symmetric when using $f$ from \eqref{sym_obj_func} : (SNMPBB))}
\label{nmpbb}
\vspace{0.2cm}
\begin{algorithmic}[1]
\label{nmpbb_alg}
\vspace{-0.2cm}
\REQUIRE $ \!\! \nabla f(W)$, $\!W_0\!=\!W^{(i)}$, $\!H_0 \!=\! {H^{(i)}}$, $\!\eta_0\!=\!\eta^{(i)}$, $ \!\theta_0 \!=\! \theta^{(i)} $, $\varphi$, $0 < \alpha_{\text{min}} < \alpha_{\text{max}}$  %$\mu$
\FOR{$k = 0,1,2,\dots$}
  \STATE Compute projected gradient step
    $$Z_k = \text{Proj}\bigl[W_k - \tfrac1\rho_0 \nabla f(W_k)\bigr],\quad \rho_0=\|H_0 H_0^T\|_2\,. $$
    
  \STATE Get descent direction $$D_k = \text{Proj}\bigl[Z_k - \alpha_k \nabla f(Z_k) \bigr] - Z_k$$ where
  $$\alpha_k = \min \bigl\{ \alpha_{\text{max}}, \max \bigl\{ \alpha_{\text{min}}, \frac{\langle S_{k-1}, S_{k-1}\rangle}{\langle S_{k-1}, Y_{k-1}\rangle} \bigr\} \bigr\}$$
  and $S_{k-1} = W_k - W_{k-1}$ and $Y_{k-1} = \nabla f(W_k) - \nabla f(W_{k-1})$

  \STATE Nonmonotone line search. Find $\beta$ 
  %with $F_k$ from \eqref{eq:func_hist} 
  s.t. $$f(Z_k + \beta D_k) \leq F_k + \beta \theta_k \langle\nabla f(Z_k), D_k\rangle$$
  where
  $$
  F_k = 
\begin{cases}
      f(W_0) & \text{if}\; k=0\\
      f(W_k) + \eta_{k-1}(F_{k-1}-f(W_k)) & \text{if}\; k\geq1\\
\end{cases}
  $$
  Then let $W_{k+1} = Z_k + \varphi \beta D_k$
\ENDFOR
%\UNTIL{Convergence criteria satisfied}
%\vspace{0.1cm}
\end{algorithmic}
\end{algorithm}
% \label{nmpbb}
%\end{algorithm}
 % \end{minipage}}
 % \end{center}

%\vspace{0.3cm}
To find $\beta > 0$ in the linesearch, typically a backtracking approach
is used. I.e., starting from $ \beta^{(0)} = 1$ the step length is successively reduced until the sufficient decrease condition in line 4 is satisfied.
Note that NMPBB is designed for the NMF, and thus for two variables. The
symmetric NMF, however, has only one variable and thus may lend itself to a 
simpler, albeit naive approach. Consider Algorithm \ref{nmpbb_alg} in
the form of a function: $ \textnormal{nmpbb}(W^{(i)},H^{(i)},H^{(i)}{H^{(i)}}^T,\eta^{(i)},\theta^{(i)},\varphi,\alpha_{\text{min}},\alpha_{\text{max}}) $.
Thus an immediate approach is to simply compute the updates as  
$ \textnormal{nmpbb}(W^{(i)}, \allowbreak W^{(i)}, \allowbreak W^{(i)}{W^{(i)}}^T, \allowbreak \eta^{(i)},\allowbreak \theta^{(i)}, \allowbreak \varphi,\alpha_{\text{min}},\alpha_{\text{max}}) $ for the symmetric factorization and set $H^{(i+1)} \gets {W^{(i+1)}}^T $ after each call. Nonetheless, this approach performs poorly in practice, possibly due to the symmetric assignments effectively reducing the degrees of freedom (i.e., fewer variables to fit the data).

% Initially, an alternating scheme was used to induce symmetry.
% After each NMPBB solve, the transpose of the newly solved $W$ or $H$ was then assigned to the old $H$ or $W$, respectively.
% Overall, this performed very poorly, likely due to each "symmetric assignment" overwriting the progress made by the NMPBB solve against the original $W$ or $H$.\\
Instead, we develop a new approach that keeps two variables but couples them to obtain symmetric factorizations. In particular, we use a quadratic penalty similar to Li et al. \cite{9606619}, and adapt it for a gradient, i.e., NMPBB type, algorithm. This means using an objective such as \eqref{sym_obj_func},
with a penalty parameter $ \lambda \ge 0 $.
%A modified solver was thus created using the symmetric penalty as described in Li et al.and the objective function in equation \ref{sym_obj_func} \cite{9606619}.
% The value of $\lambda$ depends on the size and norm of $V$; typically, for a larger matrix, $\lambda$ should be higher.
% For example, a value of 1 shows good results for $V \in R^{300 \times 300}$.
% However, this value varies and typically increases as the algorithm progresses in real implementation.\\

We use Algorithm \ref{nmpbb_alg} with objective $f$ given by \eqref{sym_obj_func} and the new gradients
\begin{gather}
\label{eq:snmpbb_grada}
    \nabla_W f(W,H^{(i)};\lambda) = W H^{(i)} {H^{(i)}}^T - V{H^{(i)}}^T + \lambda W - \lambda {H^{(i)}}^T \\
 \label{eq:snmpbb_gradb}
    \nabla_H f(W^{(i)},H;\lambda) = {W^{(i)}}^T W^{(i)} H - {W^{(i)}}^T V + \lambda H - \lambda {W^{(i)}}^T
\end{gather}
%This approach provides state-of-the-art results for certain values of $\lambda$ on synthetic data.
We call this algorithm SNMPBB. For even large matrices, SNMPBB converges quickly and to similar residual values as algorithms such as SymANLS. However, for graph clustering, SNMPBB alone converges slowly, even with special initializations. 
% In the following we use the notation that for a fixed $H_k$ 
% \begin{equation}
% \label{eq:deff}
% f(W,H_k) := f(W)
% \end{equation}
% and therefore $ \nabla_W f(W,H_k) := \nabla f(W) $.

\subsection{Graph-SNMPBB}
\label{sec:reg}
Graph clustering is an important application, since the adjacency
matrix of an undirected graph (i.e., all edges are bidirectional) is 
symmetric and nonnegative. In particular, any clustering problem may
be formulated as a graph clustering problem by preprocessing the data with a 
$K$-nearest neighbor algorithm. However, applying the NMF directly to 
a graph may yield poor performance, because the factorization is known
to result in so-called soft-clustering (Gao et al. \cite{8637461}). Nonetheless,
it has been shown that sparsity and inclusion of a graph's geometric structure
may help to increase clustering quality (Berahmand et al.~\cite{BERAHMAND2024127041}).
% Any of the standard symmetric NMF methods described earlier do not necessarily produce factorizations that have clear graph structures.
% Although the soft clustering is helpful in some contexts, it may lead to incorrect groups for some points and be a hindrance for convergence.
% That is, there is no regularization and clustering may be only basic for graphs \cite{8637461}.
% It is shown that regularization based on sparseness and the geometric structure of graphs helps to maintain clustering quality in addition to the symmetric objective \cite{BERAHMAND2024127041}.
% Therefore, Cai proposed enforcing additional orthoganilty by promoting sparseness in $H$ \cite{5674058} with the following further modified problem
Therefore, we propose to enforce additional orthogonality and sparsity with
the following objective
\begin{equation}
    \label{eq:gsnmp}
    f(W,H;\lambda,\gamma) = \frac{1}{2} \left\| V - W H \right\|^2 + \frac{\lambda}{2} \left\| W - H^T \right\|^2 + \frac{\gamma}{2} \text{tr}(HLH^T)
\end{equation}
where $L$ is the graph Laplacian derived from the problem, $L = I - D^{-1/2} V D^{-1/2}$ and $D$ is the diagonal degree matrix of $V$, i.e., $ d_{ii} = e^T_i V e_i $, $ 1 \le i \le n $. Note that the use of the graph Laplacian for the NMF has been
described before (see Cai et al. \cite{5674058}), however, the authors develop a multiplicative update rule 
% , which is typically less effective than the nonmonotone line-search in this work.
and do not explicitly enforce the symmetry of the problem.
We point out that \eqref{eq:gsnmp} introduces the additional sparsity parameter $\gamma \ge 0 $.
% sparsity parameter as outlined in the above objective function and also implement the Laplacian for regularization as outlined in Cai et al.
Importantly, \eqref{eq:gsnmp} may be 
effectively used with Algorithm \ref{nmpbb}
by developing the extended gradients
\begin{gather}
    \nabla_W f(W,H^{(i)};\lambda,\gamma) = WH^{(i)}{H^{(i)}}^T-V{H^{(i)}}^T + \lambda W - \lambda {H^{(i)}}^T\\
    \nabla_H f(W^{(i)},H;\lambda,\gamma) = {W^{(i)}}^T W^{(i)} H - {W^{(i)}}^T V + \lambda H - \lambda {W^{(i)}}^T + \gamma HL % 2 \gamma
\end{gather}
This produces Graph-SNMPBB.
The best value of $\gamma$ depends on the size and norm of $V$ as well as the type of problem. We describe parameter choices in Section \ref{sec:hyper}.
% This adapts the SNMPBB method into the Graph-SNMPBB method and converges slightly slower depending on the value of $\gamma$ and select initializations.\\
% \begin{table}[H]
%     \centering
%     \begin{tabular}{|c|c|c|c|c|}
%     \hline
%      Method & Symmetry & Sparsity & Laplacian & Initialization\\
%      \hline\hline
%      NMPBB & No & No & No & Random\\
%      \hline
%      SNMPBB & Yes & No & No & SVD+Kernel\\
%      \hline
%      Graph-SNMPBB & Yes & Yes & Yes & SVD+Kernel\\
%      \hline
% \end{tabular}\\
%     \caption{Comparison of Barzilai-Borwein based PGD methods}
% \end{table}
While these modifications 
% improve the sparsity of $H$ and 
enhance graph clustering, proper initialization and preprocessing is vital for good results.
% Graph-SNMPBB makes use of similar preprocessing for $V$ as in SymNMF.
% That is, if an SVD or randomized SVD is computationally feasible initialize $W_0$ and $H_0$ using this factorization.
% , and has greatly increased clustering accuracy as opposed to random initialization.
Like in Kuang et al. \cite{Kuang2015SymNMF}, the similarity matrix for data $V$ is computed by applying a Gaussian kernel with parameter $\sigma$ and keeping only the nearest $K$ neighbors for each node, setting all other entries to zero.
This strategy keeps the most relevant local connections for each data point in the original matrix.

\subsection{LAI-SNMPBB}
\label{sec:laisnmpbb}
Using the Low-rank Approximate Input (LAI) methodology from Hayashi et al. \cite{HaEtAl25}, we also develop an LAI-SNMPBB algorithm.
Specifically, we sketch the eigendecomposition for the input matrix $V$ and then anywhere that $V$ is used in a matrix multiply, we replace it with the thinner matrices from the decomposition.
For a random sketch $ S \in \mathbb{R}^{n \times \ell } $ we compute the thin QR factorization
$ VS = QR $ and the small eigendecomposition $ Q^T V Q = P E P^T $, where $P$ is $ \ell \times \ell $
orthogonal and $ E $ is diagonal. Defining orthonormal $U = A := QP$ and $ B := E U^T $ we approximate
\begin{equation}
\label{eq:laiapprox}
V \approx Q Q^T V Q Q^T = QPE P^T Q^T = U E U^T := AB
\end{equation}
For SNMPBB, the calculation of $VH$, for example, is replaced with $UE(U^TH)$, where $U^TH$ is formed first before left multiplying by $UE$ (avoiding a large matrix multiply).
% We found that the results of factoring large matrices or clustering large datasets were very competitive or better when compared to the best of the algorithms outlined in Hayashi et al, LAI-SymPGNCG.\\
The objective and gradients for LAI-SNMPBB are given by 
\eqref{sym_obj_func} and \eqref{eq:snmpbb_grada}, \eqref{eq:snmpbb_gradb} after replacing $V$ by $AB$.
% \begin{equation}
% \label{eq:lai_snmpbb}
% \tilde{f}(W,H,\lambda) = \frac{1}{2} \left\| AB - W H \right\|^2 + \frac{\lambda}{2} \left\| W - H^T \right\|^2
% \end{equation}

% We dissect the key components of Algorithm 2.1 to build intuition for its design. 
\subsection{Algorithmic Components and Design}
SNMPBB enhances standard projected gradient descent with three mechanisms: an adaptive Barzilai-Borwein step-size using matrices instead of vectors, a nonmonotone line search strategy, and a two-stage projection scheme.
% \jbl{Ryan, can you relate these components to the algorithm? Perhaps with the line numbers, if possible.}

In line 3 of Algorithm \ref{nmpbb_alg}, a Barzilai-Borwein (BB) step-size is used as $\alpha_k$. % as calculated in \ref{eq:bb}.
It provides an adaptive scaling that approximates Newton-like behavior at first-order cost and was originally proposed for unconstrained vector optimization. 
% The BB step-size can be derived by approximating the inverse of the Hessian $(\nabla^2f)^{-1}$ with a scaled identity matrix $\alpha I$ and choosing $\alpha$ to best fit the secant equation:
% \begin{equation}
%     s_{k-1} \approx \alpha y_{k-1}
% \end{equation}
% This is solved as a relatively cheap least-squares problem and provides curvature information from historical gradients without having to expensively compute or store the Hessian.
For matrix-valued problems like NMF, we extend the vector BB step-size using the Frobenius inner product.
In particular, we solve the least-squares problem $\min_{\alpha} \left\| S_{k-1} - \alpha Y_{k-1}\right\|_F^2$
for $ m \times r $ matrices $ S_{k-1}, Y_{k-1} $  to define
\begin{align}
    \alpha_k = \langle S_{k-1}, S_{k-1} \rangle / \langle S_{k-1}, Y_{k-1}\rangle %&\\
    %&= \left\| W_k - W_{k-1}\right\|^2_F / \langle W_k - W_{k-1}, \nabla f(W_k) - \nabla f(W_{k-1}) \rangle \nonumber
\end{align}
%where the Frobenius inner product is $\langle A, B \rangle = \sum_{i,j} A_{ij}B_{ij}$.

To prevent extreme step-sizes when subsequent iterates are very similar (numerator goes to 0) or orthogonal (denominator goes to 0), we bound $\alpha_k \in [\alpha_{\text{min}}, \alpha_\text{max}]$ in Algorithm 2.1, with

\begin{equation}
    \alpha_k = \min \bigl\{ \alpha_{\text{max}}, \max \bigl\{ \alpha_{\text{min}}, \frac{\langle S_{k-1}, S_{k-1}\rangle}{\langle S_{k-1}, Y_{k-1}\rangle} \bigr\} \bigr\}
\end{equation}

% The BB step-size adapts to local curvature automatically, typically yielding faster convergence than fixed step-sizes without the expense of Hessian computation (which would cost $O(n^2 r^2)$ flops per iteration for NMF). % \\

Separately, Algorithm 2.1 performs two distinct projection operations (lines 2-3), 
which distinguish SNMPBB from standard projected gradient.
The first trial point, $Z_k$, is defined by
\begin{equation}
    \label{eq:Zk}
    Z_k = \text{Proj}[W_k - (1/\rho_0) \nabla f(W_k)], \quad \textnormal{with} \quad   \rho_0 = \left\|H_0 H^T_0\right\|_2 % , H_k
\end{equation}
This is a standard projected gradient step with automatic step-size 
scaling $1/\rho_0$. The scaling factor $\rho_0$ is closely related to the Lipschitz constant of $\nabla _W f(W,H)$ with respect to $W$ when $H$ is fixed, i.e., % f(W, H) 
$\rho = \left\|HH^T\right\|_2 + \lambda$. Since $\lambda$ is typically small relative to $\left\|HH^T\right\|_2$, it is omitted in the computation.
Using $\rho_0$ as the step-size denominator adapts to the local curvature induced by $H$. % f(\cdot, H
% Ryan: Yes, it is the matrix 2-norm. This is discussed in the proof I believe
% \jbl{Here we have to be careful. The norm in \eqref{eq:Zk} is the matrix 2-norm ? I have modified the explanation of $\rho_k$ can you check whether this is accurate ?}
The descent direction $D_k$ is defined by
\begin{equation}
    D_k = \text{Proj}[Z_k - \alpha_k\nabla f(Z_k)] - Z_k
\end{equation}
This creates a search direction from $Z_k$ toward a second projected 
gradient point.
Each projection incorporates different information. The first projection ($Z_k$) uses the current iterate $W_k$ and Lipschitz constant $\rho_k$ which captures H's immediate influence. The second projection (along direction $D_k$) uses trial point $Z_k$ and BB step-size $\alpha_k$  which captures  curvature.
Together, these provide a quasi-Newton-like update using only first-order information. %\\
Further, standard projected gradient methods require monotone decrease with $f(W_{k+1}) < f(W_k)$ at every iteration.
While this guarantees progress, it can be overly restrictive for nonconvex functions.
SNMPBB relaxes the monotone requirement by comparing against a reference value $F_k$ rather than the immediate previous objective $f(W_k)$.
In line 4 of the algorithm, this reference value is defined recursively:
\begin{equation}
\begin{cases}
    F_0 = f(W_0)\\
    F_k = f(W_k) + \eta_{k-1}(F_{k-1} - f(W_k))&\text{for } k \geq 1
\end{cases}
\end{equation}
where $\eta_k \in [0, 1)$ controls the ``memory" of past iterates.
When $\eta_k = 0$, we recover standard monotone descent, $F_k = f(W_k)$.
When $\eta_k > 0$, $F_k$ becomes a weighted average of recent function values, allowing $f(W_{k+1}) > f(W_k)$ as long as sufficient decrease holds relative to $F_k$. We accept a step $W_{k+1} = Z_k + \varphi \beta D_k$ when
\begin{equation}
    f(W_{k+1}) \leq F_k + \beta \theta_k \langle \nabla f(Z_k), D_k\rangle
\end{equation}
where $\theta_k \in (0, 1)$ is the Armijo parameter (cf. Armijo \cite{Ar66}) and $\beta$ is found by backtracking from $\beta$ = 1. Typically, $\theta_k$ is a small constant like $10^{-3}$.
Note this Armijo-type condition is compared against $F_k$ rather than $f(W_k)$.
% Nonmonotone descent allows temporarily increasing the function value, and enables larger steps and potentially faster progress.
% The symmetric factorization objective \eqref{original_symnmf_objective} is quartic in W and highly nonconvex with many local minima, especially for graph clustering applications.
The complete update is then as follows:
\begin{equation}
    W_{k+1} = Z_k + \varphi \beta
    D_k
\end{equation}
where 
% $\beta$ was determined by the nonmonotone line search and 
$\varphi > 0$ is a fixed relaxation factor as defined in Han et al.\cite{HaNePr09}.
The relaxation factor $\varphi$ amplifies this step under the assumption that the local model is overly conservative.
The authors experimentally determined that $\varphi \approx 1.7$ was the best value of amplification.

\section{Analysis}
\subsection{Convergence}
\label{eq:conv}

The following analysis is based on $W$, and is similar for $H$. 

% In the following analysis, we assume that we are solving the gradient descent subproblem for $W$.
% For $H$, each case is symmetric.
% Comment-out theorem definition 
%\newtheorem{theorem}{Theorem}
%\newtheorem*{myremark}{Remark}

% \jbl{Ryan, please carefully update the objective in the analysis to be consistent with the objective in \eqref{sym_obj_func}. Also, please change transposes $ {\top} $ to $ T $
% to be consistent with previous notation. Right now \cref{eq:thrmobj} mixes dimensions (this will trigger a few other changes in this section)  }

\begin{theorem}
\label{thm:convergence}
Consider the symmetric NMF objective function
\begin{equation}
\label{eq:thrmobj}
f(W, H; \lambda) = \frac{1}{2}\|V - WH\|^2 + \frac{\lambda}{2}\|W - H^T\|^2
\end{equation}
where $W \in \mathbb{R}^{m \times r}_+ $ and $H \in \mathbb{R}^{r \times n}_+$. Then the convergence analysis from Li et al. \cite[Theorem 4]{sym16020154} holds
for SNMPBB, and any accumulation point of the sequence $\{W_k\}_{k\ge0}$ is a first-order stationary point.
\end{theorem}

\begin{proof}
We observe that SNMPBB based on Algorithm \ref{nmpbb_alg} can leverage the analysis for global convergence in Li et al.\cite[Theorem 4]{sym16020154}.
Their result relies fundamentally on two properties outside of the algorithm itself: convexity of the objective and Lipschitz continuity of its gradient.
In order to give a self-contained discussion we restate the convergence arguments first. 

% For clarity we summarize the argument here, in which \cite[Lemma 1]{sym16020154} directly leads to \cite[Theorem 4]{sym16020154} and therefore convergence.
In particular, for \cite[Lemma 1]{sym16020154}, the authors observe that the objective function\\ $f(W) = f(W,H^{(i)})= \frac{1}{2}\left\|V-WH^{(i)}\right\|^2$ is convex and its gradient is Lipschitz continuous.
Then, it is shown that given the scaled projected gradient for step size $\alpha$ 
\[ D_\alpha(W)\;=\;\text{Proj}\bigl[W - \alpha\nabla f(W)\bigr]\,-\,W \]
that for every $W \geq 0$,
\[ \bigl\langle\nabla f(W),\,D_\alpha(W)\bigr\rangle \;\le\; -\tfrac1\alpha\|D_\alpha(W)\|^2 \:\: \text{ and } \:\: D_\alpha(W)=0\:\:\iff\:\:W\text{ is stationary} \]
This is proved in Birgin et al. \cite{birginetal} and in \cite[Lemma 2]{sym16020154}. Furthermore, let $D_k$ equal $D_{\alpha}$ after the line search is completed.
It is shown that
\[ D_k(Z_k) = 0\;\iff\;Z_k \text{ is stationary} \]
which is in \cite[Lemma 3]{sym16020154}.
This is then combined with the established nonmonotone reference sequence $F_k$ in step 4 of algorithm \ref{nmpbb_alg} and established by straightforward algebra (using the descent bound above) that
\[ f(W_k)\le F_k\le F_{k-1}\le\cdots\le F_0, \]
so $\{F_k\}$ is nonincreasing and bounded below.
This references both \cite[Lemma 6]{sym16020154} and \cite[Theorem 1]{sym16020154}. 
This implies 
\[ f\bigl(Z_k+\beta_k D_k(Z_k)\bigr) \;\le\; F_k \;+\;\theta_k\,\beta_k\, \langle\nabla f(Z_k),D_k(Z_k)\rangle \]
and that each nonmonotone Armijo line search must terminate after finitely many backtracking steps by combining the descent estimate of the previous step with Lipschitz continuity of $\nabla f$.
The resulting step sizes $\{\beta_k\}$ satisfy \(\beta_k\ge\widetilde\beta>0\) whenever $W_{k+1}$ is not already stationary, which forces a uniform amount of decrease in the nonmonotone reference \cite[Lemma 7]{sym16020154}.
This yields $\lim_{k \rightarrow \infty} F_k = \lim_{k \rightarrow \infty} f(W_k)$ and the existence of $\delta > 0$ such that $F_k - f(W_{k+1}) \geq \delta \left\| D_k (Z_k) \right\|^2$ \cite[Lemma 8]{sym16020154}.
From this, \(\lim_{k \rightarrow \infty}\|D_k(Z_k)\|=0\) and hence any accumulation point of~$\{W_k\}$ is a first‐order stationary point \cite[Theorem 3, Theorem 4]{sym16020154}.

% Crucially, we see that this argument depends only on Lemma 1 which assumes that $f$ is convex and $\nabla f$ is Lipschitz continuous.
To apply this framework to our symmetric penalty objective (\ref{sym_obj_func}), we verify convexity and gradient Lipschitz continuity.
To show convexity for SNMPBB, we analyze the symmetric penalty term first. Since $ \left\|W-H^T\right\|^2 = \text{tr}\left((W-H^T)^T(W-H^T)\right)$, which is quadratic in $W$ when $H$ is fixed (and vice versa), this term is convex. Therefore, $f$ is still convex when this is added to $ \left\| V-WH\right\|_F^2$, since the sum of convex functions is convex \cite{GuTaLuYu12}.

To show Lipschitz continuity, we use the fact that $\nabla_W f(W) = WHH^T-VH^T + \lambda W - \lambda H^T$
\begin{gather*}
\left\| \nabla_W f(W_1) - \nabla_W f(W_2) \right\|^2\\
= \left\| W_1HH^T-VH^T+\lambda W_1 - \lambda H^T - W_2 HH^T + VH^T - \lambda W_2 + \lambda H^T \right\|^2\\
= \left\| (W_1 - W_2)H H^T + \lambda(W_1 - W_2)\right\|^2\\
= \left\| (W_1 - W_2)(H H^T + \lambda I_r)\right\|^2 \quad \text{where $I_r \in \mathbb{R}^{r \times r}$}
\end{gather*}
We let $A = HH^T + \lambda I_r$, and by the proof of Lemma 2 in \cite{GuTaLuYu12}, we can say that $\left\| (W_1 - W_2)A\right\|^2 \leq \delta_1^2\left\|W_1 - W_2\right\|^2$ where $\delta_1$ is the largest singular value of $A$.
Therefore,
\[ \left\| \nabla_W f(W_1) - \nabla_W f(W_2) \right\| \leq L \left\|W_1 - W_2\right\| \]
where $L = \delta_1 = \left\|H H^T \right\|_2 + \lambda$ since $\lambda$ shifts the eigenvalues by a positive amount.
Thus, we have shown that the symmetric objective function \eqref{eq:thrmobj} is convex and Lipschitz continuous and that thus SNMPBB converges to a stationary point according to the original argument.
\end{proof}

\subsection{Convergence properties for modifications}
When adding graph regularization, the gradient for $W$ remains unchanged since $H$ is the only element in the modified objective function.
In particular, the gradient for $H$ changes through the addition of $\gamma H$.
Since the original function is strongly convex quadratic, adding a linear term does not impact the original convexity.
Further, we show that this gradient is also Lipschitz continuous.
For the gradient w.r.t. $W$, the proof is unchanged.
But w.r.t. $H$, and using similar algebra as before,
\begin{gather*}
    \left\| \nabla_H f(H_1) - \nabla_H f(H_2) \right\|\\
    = \left\| (H_1-H_2)(W^TW + \lambda I_r + \gamma L) \right\|% \\
\end{gather*}
Again, we let $A = W^TW + \lambda I_r + \gamma L$.
The only difference is $\gamma L$; this is positive semidefinite since $\gamma > 0$ and $L$ is the normalized graph Laplacian.
Since $W^TW$ and $\gamma L$ are both positive semidefinite, their sum is, as well.
Further, when adding $\lambda I_r$ where $\lambda > 0$, the largest eigenvalue is positive; let it equal $\delta_1$.
Thus, $\left\|(H_1-H_2)A\right\|^2 \leq \delta_1^2 \left\|H_1-H_2\right\|^2$, and furthermore,
\begin{equation*}
    \left\| \nabla_H f(H_1) - \nabla_H f(H_2) \right\| \leq L_2 \left\|H_1 - H_2 \right\|
\end{equation*}
where $L_2$ is the Lipschitz constant equal to the largest eigenvalue of $W^TW + \lambda I_r + \gamma L$. %\\
Finally, for LAI-SNMPBB, convergence similarly holds, but for the low-rank version of $V$ computed through randomized sketching. % it converges to 

We restate proposition 3.1 from Hayashi et al, to obtain a general error bound:
\begin{equation}
    0 \leq \left\|V-W^*H^*\right\| - \min_{W,H} \left\|V-WH\right\| \leq 2 \left\|V-AB\right\| = 2\mu
\end{equation}
where $W^*$ and $H^*$ is the computed solution of LAI-SNMPBB, i.e. $\lim_{i \to \infty} W^{(i)}H^{(i)} = W^* H^*$.
That is, the gap between the LAI algorithm's residual for $V$ and the optimal residual is at most twice the error of approximating $V$ by $AB$, $\mu = \left\|V-AB\right\|$
% That is, the difference of the LAI algorithm's residual against the original $V$ matrix and the best possible residual against $V$ are less than twice the residual of the approximation of $V$ by $AB$.

% We invoke proposition 3.1 from Hayashi et al.for an error bound of
% \begin{equation}
%     \epsilon \leq 2\left\|V-AB\right\|_F + \min_{W,H} \left\|V-WH\right\|_F
% \end{equation}
% where $\epsilon$ is the error 
% \jbnote[layout=inline]{Ryan, spell-out what the $\epsilon$ error is. Presumably something like $ \epsilon = \| AB - W H \|  $}, $V$ is the original input matrix, $A \in \mathbb{R}^{m \times l}$ and $B \in \mathbb{R}^{l \times n}$ are the approximated factors such that $V \approx AB$, and $W \in \mathbb{R}_+^{m \times r}$ and $H \in \mathbb{R}_+^{r \times n}$ are factors for the minimum achievable error. %to calculate the minimum error achievable.

\subsection{LAI Curvature Analysis}
We now analyze how the LAI approximation affects the gradient and curvature information used by SNMPBB.
This provides theoretical context for the empirical results in Figures \ref{fig:perf_profiles} and \ref{fig:wos}.
Recall that the LAI objective replaces $V$ with a low rank approximation $AB$ where $A \in \mathbb{R}^{m \times \ell}$ and $B \in \mathbb{R}^{\ell \times n}$.
Then $\mu = \left\|V-AB\right\|$ measures the approximation quality.
For SNMPBB, the gradient with respect to $W$ with LAI approximations is
\begin{equation}
    \nabla_W f_{\text{LAI}}(W) = WHH^T - ABH^T + \lambda W - \lambda H^T
\end{equation}
Now, the bias created by the modification when compared to the original, non-LAI objective is
\begin{equation}
    \nabla_W f_{\text{LAI}}(W) - \nabla_W f(W) = (V-AB)H^T
\end{equation}
We note that, given the definition of $\mu$ and properties of the Frobenius norm, the gradient error $\epsilon$ satisfies
\begin{equation}
    \epsilon = \left\| \nabla_W f_{\text{LAI}}(W) - \nabla_W f(W) \right\| \leq \mu \left\|H\right\|
\end{equation}
for all $W$.
Further, the BB step size $\alpha_k$ computed using $f_{\text{LAI}}$ is identical to the BB step size computed from $f$:
\begin{equation}
    Y_{k-1}^{\text{LAI}} = \nabla_W f_{\text{LAI}}(W_k) - \nabla_W f_{\text{LAI}}(W_{k-1}) = (W_k - W_{k-1})(HH^T + \lambda I) = Y_{k-1}^{\text{true}}
\end{equation}
Since $H$ is held fixed here, the bias term $(V-AB)H^T$ is constant in $W$ and cancels exactly in the gradient difference.
The BB step size therefore depends only on the curvature of $f_{\text{LAI}}$ with respect to $W$, which equals that of $f$.
The quality of the BB step size remains the same as if $V$ was still full rank.
The only effect of the approximation is on the gradient bias which decreases as $l$ increases and $\mu$ approaches 0.

\section{Numerical Experiments}

All numerical experiments are performed on a 64-bit macOS laptop with 64GB RAM and an Apple M1 Max chip with 10 CPU cores.
The tests are run on MATLAB R2025a.
% \jbnote[layout=inline]{Ryan, please update with actual system details.} 
A GitHub repository contains the codes: 
\url{https://github.com/rswart0604/SNMPBB}.
%\url{https://github.com/tgkolda/cp_hifi_code}

\subsection{Overview}

%\edits{Which of these is better?}

We summarize the symmetric NMF algorithms used for the comparisons.

\begin{itemize}
    \item SymANLS \cite{Kuang2015SymNMF}: Alternating nonnegative least squares;  $O(n^2r^2)$  %\cite{}
    % \begin{itemize}
    %     \item Approach: Alternating nonnegative least squares
    %     \item Complexity: $O(n^2r^2)$
    % \end{itemize}
    % \item SymHALS
    % \begin{itemize}
    %     \item Approach: Hierarchical updates
    %     \item Complexity: $O(nr)$
    % \end{itemize}
    \item SymNewton \cite{Kuang2015SymNMF}: Constrained Newton updates; $ O(n^3r^3) $ 
    % \begin{itemize}
    %     \item Approach: Constrained Newton updates
    %     \item Complexity: $O(n^3r^3)$
    % \end{itemize}
    \item PGD \cite{zhang2023}: Proj. gradient on both $W$,$H$ simultaneously; $O(n^2 r)$
    \item LAI-SymPGNCG \cite{HaEtAl25}: Gauss-Newton with conjugate gradient solver and low-rank approximate input; $O(n^2 r^2)$
    \item SNMPBB (our): Proj. gradient with BB step, nonmnt. line search; $O(n^2r)$ 
    % \begin{itemize}
    %     \item Approach: Projected gradient with BB step, nonmonotone line search
    %     \item Complexity: $O(n^2r)$
    % \end{itemize}
    \item Graph-SNMPBB (our): Proj. gradient with BB step, nonmnt. line search and graph regularization parameters; $O(n^2r)$ 
    \item LAI-SNMPBB (our): Proj. gradient with BB step, nonmnt. line search and low-rank approximate input; $O(n \ell r)$ 

    % \begin{itemize}
    %     \item Approach: Projected gradient with BB step, nonmonotone line search, and graph regularization parameters
    %     \item Complexity: $O(n^2r)$
    % \end{itemize}
\end{itemize}

We also consider a modified PGD algorithm that uses a larger step-size for each update step as opposed to the Lipschitz estimate in the code for \cite{zhang2023}.

% \begin{table}[H]
%     \centering
%     \resizebox{\textwidth}{!}{  % Resize table to fit within text width
%     \begin{tabular}{|c|p{1.25in}|c|p{1.5in}|}
%         \hline
%          Algorithm & Approach & Complexity & Notes\\
%          \hline\hline
%          SymANLS & Alternating nonnegative least squares & $\approx O(n^2 r^2)$
%          & Very robust\\
%          \hline
%          SymHALS & Hierarchical updates & $O(nr)$ & Memory efficient, less aggressive updates\\
%          \hline
%          SymNewton & Constrained Newton updates &  $O(n^3 r^3)$ & Higher order information used\\
%          \hline
%          SNMPBB & Projected gradient with BB step, nonmonotone line search  & $O(n^2r)$ & Converges quickly on synthetic data\\
%          \hline
%          Graph-SNMPBB & Projected gradient with BB step, nonmonotone line search, and regularization & $O(n^2 r)$ & Converges well for graph data\\
%          \hline
         
%     \end{tabular}
%     }
%     \caption{A comparison of all symmetric NMF algorithms}
%     \label{tab:all_algos}
% \end{table}

\subsection{Practical Considerations for SNMPBB}
\label{sec:hyper}
There are three main hyperparameters in SNMPBB: $\lambda$, the value of $K$ for the $K$-means preprocessing, and $\gamma$.
%each of these require some degree of tuning before practical use.
A too-large value of $\lambda$ will cause the sparsity of $W$ and $H$ to increase too much; the same is true for $\gamma$.
Some guidelines for the hyperparameters are: %\\
(1) The symmetric penalty parameter, $\lambda$, should scale roughly with the size of the matrix and its magnitude.
For synthetic data, a good value is about $0.01 \cdot \left\|V\right\|$.
This ensures that the weight of the symmetric penalty is approximately equal to that of the objective function's gradient itself and does not dominate nor get ignored. %\\
(2) $\gamma$ follows a similar fashion, typically also set to $0.001 \cdot \left\|V\right\|$. %\\
(3) The value of $K$ in the $K$-means preprocessing of graph data for Graph-SNMPBB should also roughly scale with the data.
This parameter may vary widely depending on the sparsity of data and number of clusters.
Some values include for 3000 points, $k=80$; for 1000, $k=50$; and for 300, $k=15$. %\\
% Initialization for $W_0$ and $H_0$ depends on the nature .
(4) Typically, purely random initialization for $W_0$ and $H_0$ with values between 0 and 1 works well.
However, for many problems, the following methodology developed by \cite{Kuang2015SymNMF} is effective: generate a random matrix of values from 0 to 1 and then multiply this by $2 \cdot \sqrt{\zeta/r}$, where $\zeta$ is the average value of the elements of $V$ and $r$ is the approximating rank; that is, $W_0 \in \mathbb{R}^{m \times r}$.
To find $H_0$ we simply transpose $W_0$.
This is the method of initialization used for every numerical experiment.
For graph-based methods, it is best to set $W_0$ and $H_0$ to the singular value decomposition of $V$.
Assuming $V \approx UEP^T$ with $E$ as the diagonal matrix of singular values, $W_0$ is set equal to $\left|U E^{1/2} \right|$ and $H_0$ to $\left|E^{1/2} P^T\right|$.
However, due to the high cost of such initializations, they are not necessarily recommended for larger problems and not used in this paper. % \\
(5) For some problems that include very large and sparse matrices, sometimes the line search struggles to find an effective magnitude for the descent direction.
This is likely due to the lack of curvature for such problems.
To resolve this, we suggest reducing the maximum number of iterations for the line-search.

\subsection{Synthetic data}
We first validate on synthetic data.
Our goal is to demonstrate that the combination of BB step-sizes, nonmonotone line search, and the penalty-based symmetric formulation yields a method that is fast and accurate even when compared to well-known higher-order methods.

Similar as in the work of \cite{Huang2015Quadratic} and \cite{sym16020154} we provide a comparison for dense synthetic data first. For various sizes and ranks we initialize a matrix $W \in R^{n \times r}_+$ and compute $V = WW^T$; the residual is $\left\| V - WH \right\| / \left\| V \right\|$.
The convergence of the SNMPBB, Symmetric Alternating Nonnegative Least Squares (ANLS), and Symmetric Newton algorithms are presented in table \ref{tab:synthetic}.

\begin{table}[tbp]
    \centering
    \begin{tabular}{|cc|c|c|c|}
    \hline
         m & r & Algorithm & Time & Residual\\
         \hline
         \hline
         100 & 2 & SNMPBB & \num{0.003943} & \num{0.1318}\\
          &  & ANLS & \num{0.009971} & \num{0.1327}\\
          &  & Newton & \num{0.006659} & \num{0.2398}\\
         \hline
         100 & 5 & SNMPBB & \num{0.009257} & \num{0.07844}\\
          &  & ANLS & \num{0.1228} & \num{0.0953}\\
          &  & Newton & \num{0.01083} & \num{0.1042}\\
         \hline
         100 & 20 & SNMPBB & \num{0.09291} & \num{0.01333}\\
          &  & ANLS & \num{0.5228} & \num{0.1629}\\
          &  & Newton & \num{0.03482} & \num{0.0952}\\
         \hline
         200 & 2 & SNMPBB & \num{0.004911} & \num{0.1412}\\
          &  & ANLS & \num{0.069} & \num{0.1413}\\
          &  & Newton & \num{0.0157} & \num{0.1479}\\
         \hline
         200 & 5 & SNMPBB & \num{0.02336} & \num{0.08287}\\
          &  & ANLS & \num{0.2094} & \num{0.08848}\\
          &  & Newton & \num{0.03387} & \num{0.09047}\\
         \hline
         200 & 20 & SNMPBB & \num{0.08474} & \num{0.01533}\\
          &  & ANLS & \num{1.084} & \num{0.1338}\\
          &  & Newton & \num{0.09246} & \num{0.08643}\\
         \hline
         300 & 2 & SNMPBB & \num{0.009745} & \num{0.1438}\\
          &  & ANLS & \num{0.147} & \num{0.1439}\\
          &  & Newton & \num{0.03628} & \num{0.159}\\
         \hline
         300 & 5 & SNMPBB & \num{0.03659} & \num{0.08066}\\
          &  & ANLS & \num{0.5641} & \num{0.08965}\\
          &  & Newton & \num{0.07296} & \num{0.08824}\\
         \hline
         300 & 20 & SNMPBB & \num{0.1217} & \num{0.01803}\\
          &  & ANLS & \num{1.776} & \num{0.1233}\\
          &  & Newton & \num{0.2007} & \num{0.08645}\\
         \hline
         500 & 2 & SNMPBB & \num{0.02695} & \num{0.146}\\
          &  & ANLS & \num{0.1549} & \num{0.146}\\
          &  & Newton & \num{0.1034} & \num{0.1461}\\
         \hline
         500 & 5 & SNMPBB & \num{0.1158} & \num{0.07888}\\
          &  & ANLS & \num{0.8276} & \num{0.08245}\\
          &  & Newton & \num{0.2273} & \num{0.08085}\\
         \hline
         500 & 20 & SNMPBB & \num{0.2158} & \num{0.01769}\\
          &  & ANLS & \num{2.871} & \num{0.1189}\\
          &  & Newton & \num{0.6353} & \num{0.1553}\\
         \hline
         1000 & 2 & SNMPBB & \num{0.06556} & \num{0.1516}\\
          &  & ANLS & \num{0.2692} & \num{0.1516}\\
          &  & Newton & \num{0.4412} & \num{0.2483}\\
         \hline
         1000 & 5 & SNMPBB & \num{0.4346} & \num{0.0845}\\
          &  & ANLS & \num{1.218} & \num{0.08463}\\
          &  & Newton & \num{0.7044} & \num{0.2776}\\
         \hline
         1000 & 20 & SNMPBB & \num{0.6035} & \num{0.01774}\\
          &  & ANLS & \num{6.044} & \num{0.09781}\\
          &  & Newton & \num{2.246} & \num{0.1333}\\
         \hline
    \end{tabular}
\caption{Runtime and relative residual comparisons on dense synthetic symmetric NMF problems.
For each experiment, a random matrix $W \in \mathbb{R}_+^{n \times r}$ is used to form the symmetric input $V = WW^T$.
Across all settings, SNMPBB achieves residuals comparable to or lower than SymANLS \cite{Kuang2015SymNMF} and SymNewton \cite{Kuang2015SymNMF} while requiring less computational time.
The largest advantages occur for higher-rank factorizations where scalability benefits of first-order gradient methods become most apparent.}
\label{tab:synthetic}
\end{table}

The results of table $\ref{tab:synthetic}$ show that SNMPBB converges to similar tolerances as ANLS \cite{Kuang2015SymNMF} and SymNewton \cite{Kuang2015SymNMF} in as much or less time.
Notably, as the rank $r$ grows, the time to convergence for ANLS and SymNewton grows much larger than for SNMPBB while the residual remains similar.
This trend is true of even smaller problems.
While the value of $m$ has some influence on the time, $r$ has a more notable effect.
This shows that SNMPBB may perform better for larger problems with high-rank settings. % \\

% These results challenge the prevailing view that projected gradient methods are unsuitable for symmetric NMF.
% While vanilla PGD indeed performs poorly on symmetric problems, SNMPBB demonstrates that strategic enhancements such as using alternating methods with BB step-sizes to capture local curvature and penalty-based symmetric coupling can turn gradient descent into a competitive method.
% With appropriate mechanisms, SNMPBB achieves SymANLS-level accuracy at a fraction of the cost, particularly as rank increases where the $O(n^2r^2)$ vs. $O(n^2r)$ complexity difference becomes pronounced.

\subsection{Geometric graph clustering}
For graph clustering, we first compare Graph-SNMPBB to SNMPBB, SymANLS, PGD, and SymNewton on a bullseye dataset inspired by Kuang et al. \cite{Kuang2015SymNMF}. % \\
From Fig. \ref{fig:all_bullseye_datasets}, the convergence is clearly much faster for Graph-SNMPBB than regular SNMPBB or even SymANLS.
For larger numbers of points the value of $K$ in the nearest neighbors preprocessing is increased 
(cf. the discussion in Sec. \ref{sec:hyper}).
% This leads to the quicker speed for Graph-SNMPBB and takes up a significant portion of its time to convergence.
Furthermore, Graph-SNMPBB and SNMPBB may take advantage of sparsity when computing gradients and especially when using the Laplacian.
Visual graph data for the 3000 point bullseye is shown in the two right plots of figure \ref{fig:all_bullseye_datasets}; the nonlinearity of the bullseyes is pertinent to the use case of symmetric NMF.
%The quadrants themselves are more linearly separable and therefore standard SNMPBB performs closer to the Graph-SNMPBB.
%Furthermore, standard SNMPBB converges fully for each plot. % \\

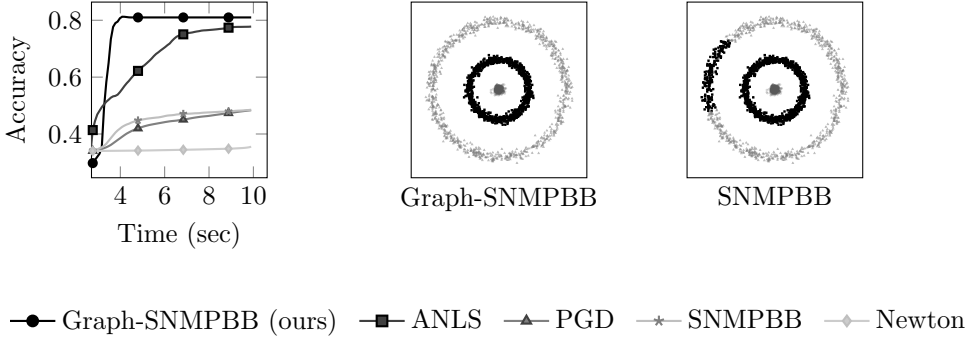
\begin{figure}[htbp]
    \label{fig:geo_clustering}
        %\vspace{0.2cm}
    \begin{subfigure}[t]{0.3\textwidth}
        %\vspace{0.2cm}
        \begin{tikzpicture}[baseline=(current bounding box.north)]
        \begin{axis}[
            %title = Graph-SNMPBB,
            xlabel={Time (sec)}, % 
            ylabel={Accuracy},
            grid=none,
            width=\linewidth,
            height=\linewidth, % 6cm
            %no markers,
            mark repeat=20,
            mark size=1.5pt,
            xmin=2.7
        ]
        \addplot+[mark=*, color=black, thick, mark options={fill=black, draw=black}] table[x index=0, y index=1, col sep=comma] {data/3000_bullseye.csv};
        \addplot+[mark=square*, color=darkgray, thick, mark options={fill=darkgray, draw=black}] table[x index=0, y index=3, col sep=comma] {data/3000_bullseye.csv};
        \addplot+[mark=triangle*, color=gray, thick, mark options={fill=gray, draw=darkgray}] table[x index=0, y index=4, col sep=comma] {data/3000_bullseye.csv};
        \addplot+[mark=star, color=lightgray, thick, mark options={fill=lightgray, draw=gray}] table[x index=0, y index=2, col sep=comma] {data/3000_bullseye.csv};
        \addplot+[mark=diamond*, color=black!20, thick, mark options={fill=black!20, draw=lightgray}] table[x index=0, y index=5, col sep=comma] {data/3000_bullseye.csv};
        \end{axis}
        \end{tikzpicture}
        %\caption{3000 Bullseye Dataset}
        %\label{fig:3000_bullseye}
    \end{subfigure}
    \hfill
    \begin{subfigure}[t]{0.3\textwidth}
        \begin{tikzpicture}[baseline=(current bounding box.north)]
            \begin{axis}[
                xlabel={Graph-SNMPBB},
                grid=none,
                width=\linewidth,
                height=\linewidth,
                scaled ticks = false,
                xticklabel = \empty,
                yticklabel = \empty,
                xtick = \empty,
                ytick = \empty,
            ]
            \addplot[only marks, mark=*, mark size=0.25pt, color=black!65,opacity=0.25]         % blue
                table[x index=0, y index=1, col sep=comma] % adjust indices
                {data/gs_cluster_1.txt};
            \addplot[only marks, mark=triangle, mark size=0.25pt, color=gray,opacity=0.5]    % red
                table[x index=0, y index=1, col sep=comma] % adjust indices
                {data/gs_cluster_2.txt};
            \addplot[only marks, mark=square, mark size=0.25pt, color=black,opacity=1]      % green
                table[x index=0, y index=1, col sep=comma] % adjust indices
                {data/gs_cluster_3.txt};
            \end{axis}
        \end{tikzpicture}
    \end{subfigure}%
    \begin{subfigure}[t]{0.3\textwidth}
    \begin{tikzpicture}[baseline=(current bounding box.north)]
        \begin{axis}[
            xlabel={SNMPBB},
            grid=none,
            width=\linewidth,
            height=\linewidth,
            scaled ticks = false,
            xticklabel = \empty,
            yticklabel = \empty,
            xtick = \empty,
            ytick = \empty,
        ]
        \addplot[only marks, mark=*, mark size=0.25pt, color=black!65,opacity=0.25]         % blue
            table[x index=0, y index=1, col sep=comma] % adjust indices
            {data/s_cluster_1.txt};
        \addplot[only marks, mark=square, mark size=0.25pt, color=black,opacity=1]      % green
            table[x index=0, y index=1, col sep=comma] % adjust indices
            {data/s_cluster_2.txt};
        \addplot[only marks, mark=triangle, mark size=0.25pt, color=gray, opacity=0.5]   % red
            table[x index=0, y index=1, col sep=comma] % adjust indices
            {data/s_cluster_3.txt};
        \end{axis}
    \end{tikzpicture}
    \end{subfigure}%
\vspace{0.5cm}
\begin{tikzpicture}
\begin{axis}[
            hide axis,
            xmin=0, xmax=1,
            ymin=0, ymax=1,
            legend columns=5,
            legend style={
                at={(0.8,1.05)}, % 0.5
                anchor=south,
                draw=none,
                /tikz/every even column/.append style={column sep=0.2cm}
            },
        ]
\addlegendimage{black, thick, mark=*}
\addlegendentry{Graph-SNMPBB (ours)}
\addlegendimage{darkgray, thick, mark=square*,mark options={fill=darkgray, draw=black}}
\addlegendentry{ANLS}
\addlegendimage{gray, thick, mark=triangle*,mark options={fill=gray, draw=darkgray}}
\addlegendentry{PGD}
\addlegendimage{lightgray, thick, mark=star,
mark options={fill=lightgray, draw=gray}}
\addlegendentry{SNMPBB}
\addlegendimage{black!20, thick, mark=diamond*,
mark options={fill=black!20, draw=lightgray}}
\addlegendentry{Newton}
\end{axis}
\end{tikzpicture}
\centering
\caption{Convergence and clustering results on a concentric bullseye dataset ($n=3000$ points, $r=3$ clusters, $K=80$ nearest neighbors).
Left: clustering accuracy vs.\ wall-clock time for all five algorithms averaged over 10 random initializations.
Right: final cluster assignments produced by Graph-SNMPBB (left) and SNMPBB (right).
Graph-SNMPBB achieves higher clustering accuracy in less wall-clock time than competing methods while correctly recovering the nonlinear ring structure more often.
 }
\label{fig:all_bullseye_datasets}
\end{figure}

\subsection{Higher-dimensional graph clustering}
For other clustering tasks, we use the following real world datasets:
\begin{enumerate}
    \item COIL20\footnote{\urlstyle{tt}\href{https://www.cs.columbia.edu/CAVE/software/softlib/coil-20.php}{https://www.cs.columbia.edu/CAVE/software/softlib/coil-20.php}}, or Columbia Object Image Library with 20 objects, is an image library from Columbia University that contains 1440 normalized images of toys with the size of each image being 128 by 128 pixels of grayscale.
    The background has been discarded from each image and the object is fit to the border.
    For each object, each image was taken 5 degrees apart as the object was rotated around its center, giving 72 images for each object.
    \item  Isolet1\footnote{\label{foot1}\urlstyle{tt}\href{http://www.cad.zju.edu.cn/home/dengcai/Data/MLData.html}{http://www.cad.zju.edu.cn/home/dengcai/Data/MLData.html}}: This dataset contains spoken letter recognition data.
    30 subjects said each letter of the alphabet twice.
    There are 1560 examples and each example has 617 features.
    \item MNIST Train\footnotemark[2] contains 128 by 128 pixel images of handwritten digits in grayscale.
    The training set uses 2000 images and 10 classes (one for each digit) split evenly.
    \item  The ORL\footnote{\urlstyle{tt}\href{https://www.cl.cam.ac.uk/research/dtg/attarchive/facedatabase.html}{https://www.cl.cam.ac.uk/research/dtg/attarchive/facedatabase.html}} (Olivetti Research Laboratory) face database contains 40 different people whose faces were taken photos of at various times with various lightings and facial details/expressions.
    There are 10 different normalized 32 by 32 pixel grayscaled images for each person, resulting in 400 examples and 40 classes.
    \item Reuters-21578\footnote{\urlstyle{tt}\href{https://www.daviddlewis.com/resources/testcollections/reuters21578/}{https://www.daviddlewis.com/resources/testcollections/reuters21578/}} is a popular collection of texts from the Reuters newswire that categorizes financial news across categories.
    The set contains 8293 different articles with 18933 words; our analysis uses the 20 largest clusters with most articles, resulting in 7800 total articles and 20 classes.
    \item The TDT2\footnote{\urlstyle{tt}\href{http://www.cad.zju.edu.cn/home/dengcai/Data/TextData.html}{http://www.cad.zju.edu.cn/home/dengcai/Data/TextData.html}} set is a collection of word documents (from the first half of 1998) including 2 newswires, 2 television programs, and 2 radio programs.
    We again only use the term documents from the top 20 clusters by size, resulting in 8741 documents with 36771 words in the encoding.
\end{enumerate}

Figure \ref{fig:all_datasets} shows, averaged over 10 runs, the results of different algorithms' accuracy versus time when clustering these datasets.
We found that the clustering methodology from \cite{Kuang2015SymNMF} worked much better for many algorithms than the preprocessing steps described earlier for the geometric clustering. Every algorithm in the figure uses this different processing.
The same random initialization was used for each algorithm with the methodology described earlier.
Furthermore, we do not show the results for either  SNMPBB without graph regularization or Newton methods due to very poor performance in clustering.
This underpins the importance of the graph regularization in the Graph-SNMPBB algorithm. %\\
Graph-SNMPBB demonstrates consistently strong performance across all six datasets, achieving either the best or near-best clustering accuracy in significantly less time than competing methods.
Compared to specifically ANLS, it matches or exceeds its final accuracy on five of six datasets while still converging substantially faster.
These speed advantages are more pronounced on larger, sparser datasets like Reuters and TDT2 (where the computational benefits of the $O(n^2 r)$ complexity versus $O(n^2 r^2)$ for ANLS become evident). %\\
The relative performance varies with dataset characteristics. On image datasets with strong visual structure such as COIL20 or MNIST, Graph-SNMPBB shows the largest advantage.
For text datasets like Reuters or TDT2, both ANLS and Graph-SNMPBB achieve high accuracies, but Graph-SNMPBB still converges faster.
% On the smaller ORL dataset (400 examples), Modified PGD's aggressive step-size appears sufficient for rapid convergence, while Graph-SNMPBB's regularization may impose unnecessary constraints for this relatively simple clustering task.

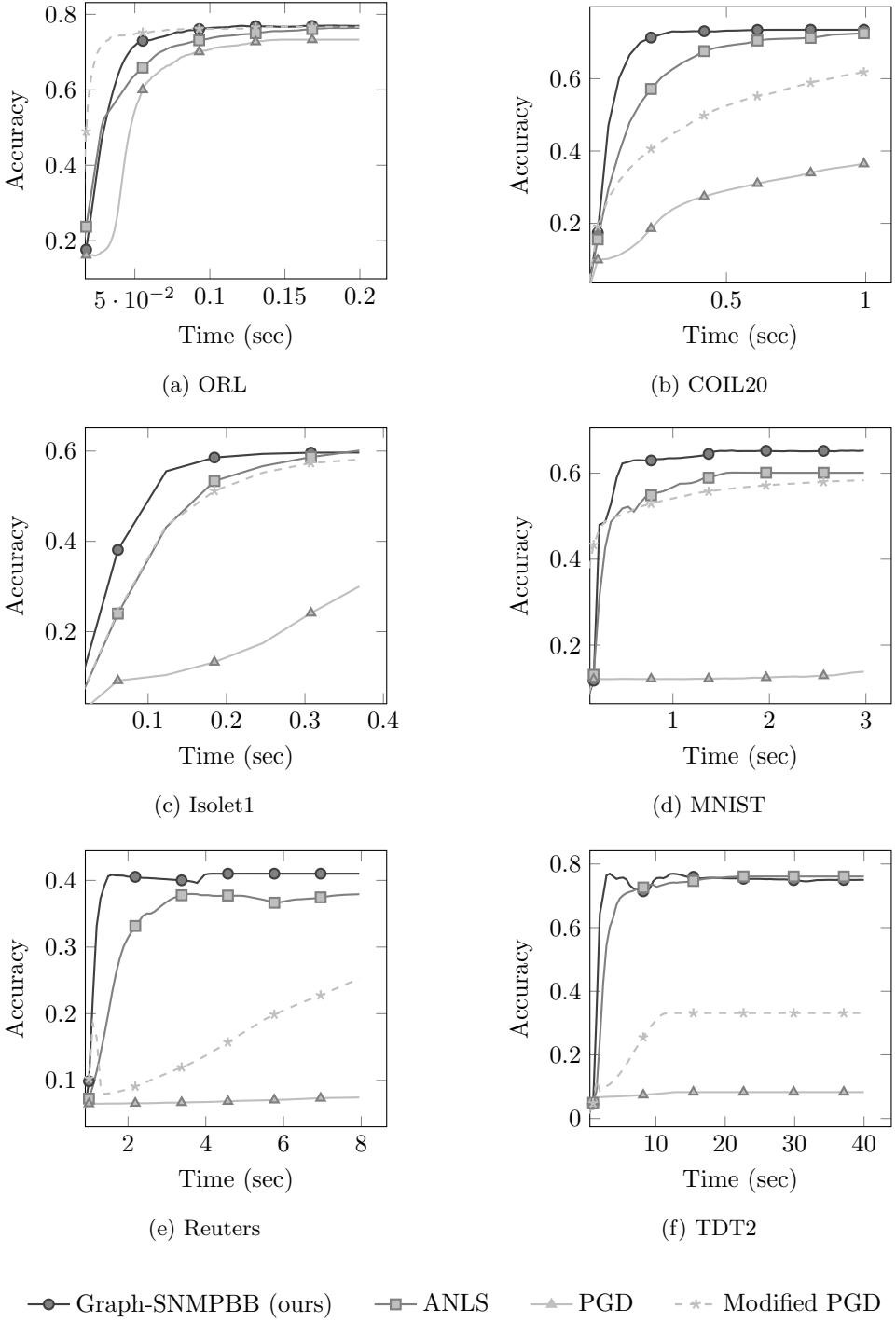
\begin{figure}[htbp]
\centering
\begin{subfigure}[b]{0.45\textwidth}
\begin{tikzpicture}
\begin{axis}[
    xlabel={Time (sec)},
    ylabel={Accuracy},
    grid=none,
    width=\linewidth,
    height=5.5cm,
    % no markers,
    xmin=.017,
    scaled ticks = false,
    mark repeat=30,
]
\addplot+[mark=*, color=darkgray, thick,mark options={fill=gray, draw=darkgray}] table[x index=0, y index=1, col sep=comma] {data/orl.csv};
\addplot+[mark=square*, color=gray, thick,mark options={fill=lightgray, draw=gray}] table[x index=0, y index=2, col sep=comma] {data/orl.csv};
\addplot+[mark=triangle*, color=lightgray, thick,mark options={fill=lightgray, draw=gray}] table[x index=0, y index=3, col sep=comma] {data/orl.csv};
\addplot+[mark=star, color=lightgray, thick,dashed,mark options={fill=black!20, draw=lightgray}] table[x index=0, y index=4, col sep=comma] {data/orl.csv};
\end{axis}
\end{tikzpicture}

\caption{ORL}
\label{fig:orl}
\end{subfigure}
\hfill
\begin{subfigure}[b]{0.45\textwidth}
\begin{tikzpicture}
\begin{axis}[
    xlabel={Time (sec)},
    ylabel={Accuracy},
    grid=none,
    width=\linewidth,
    height=5.5cm,
    xmin=0.01,
    mark repeat=5
]
\addplot+[mark=*, color=darkgray, thick,mark options={fill=gray, draw=darkgray}] table[x index=0, y index=1, col sep=comma] {data/coil20.csv};
\addplot+[mark=square*, color=gray, thick,mark options={fill=lightgray, draw=gray}] table[x index=0, y index=2, col sep=comma] {data/coil20.csv};
\addplot+[mark=triangle*, color=lightgray, thick,mark options={fill=lightgray, draw=gray}] table[x index=0, y index=3, col sep=comma] {data/coil20.csv};
\addplot+[mark=star, color=lightgray, thick,dashed,mark options={fill=black!20, draw=lightgray}] table[x index=0, y index=4, col sep=comma] {data/coil20.csv};

\end{axis}
\end{tikzpicture}
\caption{COIL20}
\label{fig:coil20}
\end{subfigure}

% \vspace{0.2cm}

\begin{subfigure}[b]{0.45\textwidth}
\begin{tikzpicture}
\begin{axis}[
    xlabel={Time (sec)},
    ylabel={Accuracy},
    grid=none,
    width=\linewidth,
    height=5.5cm,
    mark repeat=2,
    xmin=0.02
]
\addplot+[mark=*, color=darkgray, thick,mark options={fill=gray, draw=darkgray}] table[x index=0, y index=1, col sep=comma] {data/isolet1.csv};
\addplot+[mark=square*, color=gray, thick,mark options={fill=lightgray, draw=gray}] table[x index=0, y index=2, col sep=comma] {data/isolet1.csv};
\addplot+[mark=triangle*, color=lightgray, thick,mark options={fill=lightgray, draw=gray}] table[x index=0, y index=3, col sep=comma] {data/isolet1.csv};
\addplot+[mark=star, color=lightgray, thick,dashed,mark options={fill=black!20, draw=lightgray}] table[x index=0, y index=4, col sep=comma] {data/isolet1.csv};
\end{axis}
\end{tikzpicture}
\caption{Isolet1}
\label{fig:isolet1}
\end{subfigure}
\hfill
\begin{subfigure}[b]{0.45\textwidth}
\begin{tikzpicture}
\begin{axis}[
    xlabel={Time (sec)},
    ylabel={Accuracy},
    grid=none,
    width=\linewidth,
    height=5.5cm,
    mark repeat=10,
    xmin=0.14
]
\addplot+[mark=*, color=darkgray, thick,mark options={fill=gray, draw=darkgray}] table[x index=0, y index=1, col sep=comma] {data/2k2k.csv};
\addplot+[mark=square*, color=gray, thick,mark options={fill=lightgray, draw=gray}] table[x index=0, y index=2, col sep=comma] {data/2k2k.csv};
\addplot+[mark=triangle*, color=lightgray, thick,mark options={fill=lightgray, draw=gray}] table[x index=0, y index=3, col sep=comma] {data/2k2k.csv};
\addplot+[mark=star, color=lightgray, thick,dashed,mark options={fill=black!20, draw=lightgray}] table[x index=0, y index=4, col sep=comma] {data/2k2k.csv};

\end{axis}
\end{tikzpicture}
\caption{MNIST}
\label{fig:2k2k}
\end{subfigure}

% \vspace{0.2cm}

\begin{subfigure}[b]{0.45\textwidth}
\begin{tikzpicture}
\begin{axis}[
    xlabel={Time (sec)},
    ylabel={Accuracy},
    grid=none,
    width=\linewidth,
    height=5.5cm,
    mark repeat=12,
    xmin=0.9
]
\addplot+[mark=*, color=darkgray, thick,mark options={fill=gray, draw=darkgray}] table[x index=0, y index=1, col sep=comma] {data/reuters.csv};
\addplot+[mark=square*, color=gray, thick,mark options={fill=lightgray, draw=gray}] table[x index=0, y index=2, col sep=comma] {data/reuters.csv};
\addplot+[mark=triangle*, color=lightgray, thick,mark options={fill=lightgray, draw=gray}] table[x index=0, y index=3, col sep=comma] {data/reuters.csv};
\addplot+[mark=star, color=lightgray, thick,dashed,mark options={fill=black!20, draw=lightgray}] table[x index=0, y index=4, col sep=comma] {data/reuters.csv};
\end{axis}
\end{tikzpicture}
\caption{Reuters}
\label{fig:reuters}
\end{subfigure}
\hfill
\begin{subfigure}[b]{0.45\textwidth}
\begin{tikzpicture}
\begin{axis}[
    xlabel={Time (sec)},
    ylabel={Accuracy},
    grid=none,
    width=\linewidth,
    height=5.5cm,
    mark repeat=15,
    xmin=.5
]
\addplot+[mark=*, color=darkgray, thick,mark options={fill=gray, draw=darkgray}] table[x index=0, y index=1, col sep=comma] {data/tdt2.csv};
\addplot+[mark=square*, color=gray, thick,mark options={fill=lightgray, draw=gray}] table[x index=0, y index=2, col sep=comma] {data/tdt2.csv};
\addplot+[mark=triangle*, color=lightgray, thick,mark options={fill=lightgray, draw=gray}] table[x index=0, y index=3, col sep=comma] {data/tdt2.csv};
\addplot+[mark=star, color=lightgray, thick,dashed,mark options={fill=black!20, draw=lightgray}] table[x index=0, y index=4, col sep=comma] {data/tdt2.csv};

\end{axis}
\end{tikzpicture}
\caption{TDT2}
\label{fig:tdt2}
\end{subfigure}

\vspace{5pt}
\begin{tikzpicture}
\begin{axis}[
            hide axis,
            xmin=0, xmax=1,
            ymin=0, ymax=1,
            legend columns=5,
            legend style={
                at={(0.8,1.05)}, % 0.5
                anchor=south,
                draw=none,
                /tikz/every even column/.append style={column sep=0.5cm}
            },
        ]
\addlegendimage{mark=*, color=darkgray, thick, mark options={fill=gray}}
\addlegendentry{Graph-SNMPBB (ours)}
\addlegendimage{mark=square*, color=gray, thick, mark options={fill=lightgray}}
\addlegendentry{ANLS}
\addlegendimage{mark=triangle*, color=lightgray, thick, mark options={fill=lightgray}}
\addlegendentry{PGD}
\addlegendimage{mark=star, color=lightgray, thick, dashed, mark options={fill=black!20}}
\addlegendentry{Modified PGD}
\end{axis}
\end{tikzpicture}
\caption{Clustering accuracy vs.\ wall-clock time on six benchmark datasets, averaged over 10 random initializations.
All methods use the similarity-matrix preprocessing of Kuang et al.\ \cite{Kuang2015SymNMF}.
Graph-SNMPBB matches or exceeds the final accuracy of SymANLS on five of six datasets while converging substantially faster, with the largest advantages on the high-dimensional sparse text Reuters-21578 and TDT2, where its $O(n^2 r)$ per-iteration cost versus $O(n^2 r^2)$ for SymANLS is most pronounced.
(See also Figure~\ref{fig:reuters_motivation} for an enlarged view of MNIST.)}
\label{fig:all_datasets}
\end{figure}

\subsection{Low-rank Approximate Input comparisons}
Using the Low-rank Approximate Input (LAI) methodology from Hayashi et al. \cite{HaEtAl25},
we compare LAI-SNMPBB to LAI-SymPGNCG (the best of the algorithms outlined in Hayashi et al).
% we also develop an LAI-SNMPBB algorithm.
% Specifically, we sketch the singular value decomposition for the input matrix $V$ using the same random sketching methods and then anywhere that $V$ is used in a matrix multiply, we replace it with the thinner matrices from the decomposition.
% For SNMPBB, the calculation of $VH$, for example, is replaced with $UE(U^TH)$, where $U^TH$ is calculated before left multiplying by $UE$, avoiding a large matrix multiply.
% We found that the results of factoring large matrices or clustering large datasets were very competitive or better when compared to the best of the algorithms outlined in Hayashi et al, LAI-SymPGNCG.%\\
We tested both algorithms on 34 symmetric matrices from SuiteSparse (Davis and Hu \cite{DaHu2011}), selected by a minimum density threshold (nonzeros divided by total entries).
The benchmark set spans a wide range of problem sizes and application domains, including structural engineering (Boeing/nasa2910, Rothberg/struct4), graph combinatorics (Gset, Mycielski), and biological gene networks (Belcastro/human\_gene1, Belcastro/mouse\_gene).
Matrix sizes range from $n=1000$ to $n=49,151$ with nonzero counts ranging from approximately $4 \times 10^4$ to $2.4 \times 10^9$ (Mycielski/mycielskian16).
Given that several matrices exceed $n=14,000$ with hundreds of millions of nonzeros, this experiment was conducted at a large scale.
% We tested both algorithms on a number of matrices from the SuiteSparse repository.
% Only matrices over a certain threshold of ``density", measured by the number of nonzeros divided by the number of elements in the entire matrix, were chosen, leading to 34 total matrices to factor.
The performance profiles (Mor\'{e} and Dolan, \cite{DoMo02}) for both time to convergence and final residual are presented in figure \ref{fig:perf_profiles}.
Although the two algorithms employ different internal update strategies, both are evaluated on the same residual metric $\left\|V-WH^T\right\|/\left\|V\right\|$ and share an identical stopping criterion.
% The performance profile therefore provides a fair, objective-agnostic comparison across the 34 SuiteSparse benchmark matrices.
We can see that SNMPBB outperforms SymPGNCG in both, almost always achieving a better residual in less time.

We also observed that LAI-SNMPBB performed substantially better when the max number of iterations for each inner gradient descent subproblem was limited to about 3-5 iterations.
This is significantly fewer than the iterations required to reach full convergence tolerance for the SuiteSparse matrices in particular.
This behavior aligns closely with theoretical results on inexact iterative regularization by Molinari et al. \cite{Molinari2024}. The authors' framework for inexact proximal operators shows that controlled inexactness in inner computations can act as an implicit regularization mechanism when using appropriate early stopping.
In our LAI setting, the low-rank approximation error $\mu = \left\|V-AB\right\|$ is similar to their proximal inexactness parameter, and limiting inner iterations prevents the algorithm to overcommitting to the approximate low-rank surrogate $AB$ rather than $V$.
Specifically, when inner solves converge to high accuracy, the alternating minimization may fit the approximation error rather than the true underlying structure.
By contrast, inexact inner solves allow the algorithm to make progress in directions informed by the low-rank approximation without fully converging to solutions that encode the approximation error.

% We believe that this performs much better and finds a consistently better residual thatn SymPGNCG due to the approximation error created by the LAI problem.
% If the gradient descent subproblem is allowed to converge very finely, it may converge on a solution that is far away from the original $V$ matrix and too close to the low-rank approximated $AB \approx V$ used as input and accumulate errors over time in the alternating framework.
% Further, if instead an inexact solution is found like in some approximate Newton solvers, the algorithm avoids converging too closely to the approximated solution and takes large steps in a likely correct direction before stopping and alternating.
% This suggests that controlled inexactness in the inner solves can act as an implicit regularization mechanism, improving robustness to approximation error and preventing overfitting to the low-rank surrogate.
% As a result, limiting these inner iteration appears to be a key factor in achieving both faster convergence and improved solution quality in practice.

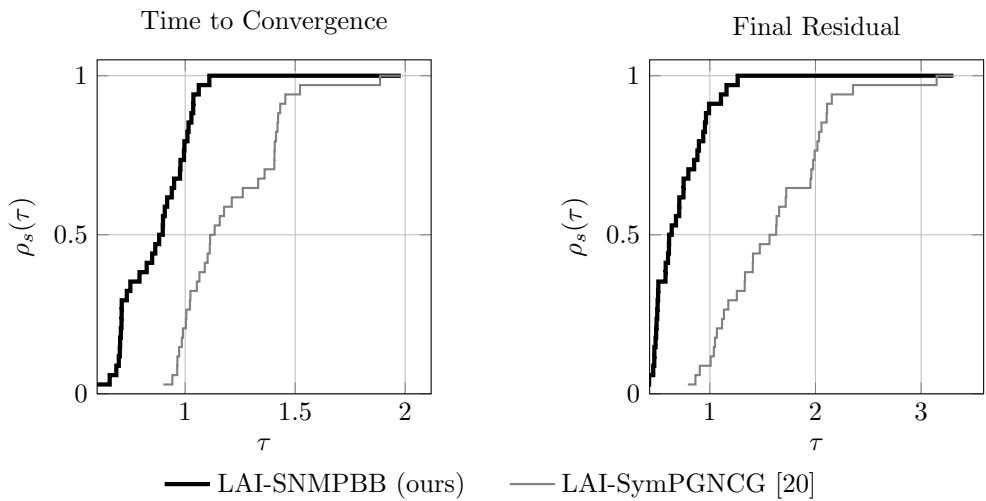
\begin{figure}[h]
    \centering
    \begin{tikzpicture}
    \begin{axis}[
        xlabel={$\tau$},
        ylabel={$\rho_s(\tau)$},
        legend pos=south east,
        grid=major,
        ymin=0, ymax=1.05,
        xmin=0.6,
        width=6cm, height=6cm,
        title={Time to Convergence},
    ]
    \addplot[black, ultra thick, const plot]
        table[x index=0, y index=1, col sep=space]{data/snmpbb_time.csv};
    % \addlegendentry{LAI-SNMPBB}
    \addplot[gray, thick, const plot]
        table[x index=0, y index=1, col sep=space]{data/pgncg_time.csv};
    % \addlegendentry{LAI-SymPGNCG}
    \end{axis}
    \end{tikzpicture}
    \hfill
    \begin{tikzpicture}
    \begin{axis}[
        xlabel={$\tau$},
        ylabel={$\rho_s(\tau)$},
        legend pos=south east,
        grid=major,
        xmin=0.43,
        ymin=0, ymax=1.05,
        width=6cm, height=6cm,
        title={Final Residual},
    ]
    \addplot[black, ultra thick, const plot]
        table[x index=0, y index=1, col sep=space]{data/snmpbb_res.csv};
    % \addlegendentry{LAI-SNMPBB}
    \addplot[gray, thick, const plot]
        table[x index=0, y index=1, col sep=space]{data/pgncg_res.csv};
    % \addlegendentry{LAI-SymPGNCG}
    \end{axis}
    \end{tikzpicture}
    \vspace{6.5pt}
    \begin{tikzpicture}
    \begin{axis}[
                hide axis,
                xmin=0, xmax=1,
                ymin=0, ymax=1,
                legend columns=5,
                legend style={
                    at={(0.8,1.05)}, % 0.5
                    anchor=south,
                    draw=none,
                    /tikz/every even column/.append style={column sep=0.5cm}
                },
                no markers
            ]
    \addlegendimage{color=black,ultra thick}
    \addlegendentry{LAI-SNMPBB (ours)}
    \addlegendimage{color=gray, thick}
    \addlegendentry{LAI-SymPGNCG \cite{10.1145/3340531.3412034}}
    \end{axis}
    \end{tikzpicture}
    
    \caption{Performance profiles for LAI-SNMPBB and LAI-SymPGNCG on 34 SuiteSparse matrices (selected by fill density; rank $r=10$, inner iterations capped at 3-5 for LAI-SNMPBB).
    The horizontal axis $\tau$ is the performance ratio relative to the best method on each problem; $\rho_s(\tau)$ is the fraction of problems for which algorithm $s$ achieves a ratio within $\tau$ of the best.
    A curve that is higher and further left dominates.
    LAI-SNMPBB dominates in both time to convergence (left) and final residual $\|V - WH\|/\|V\|$ (right), achieving a lower residual on approximately 70\% of problems at $\tau = 1$.}
    \label{fig:perf_profiles}
\end{figure}

Finally, we use the same dense clustering example as in \cite{HaEtAl25}, i.e. the Web of Science dataset \cite{kowsari2017HDLTex}.
This is a set of 46985 term documents that represent academic abstracts with 7 broad categories of scientific disciplines used as labels.
We preprocessed using Hypergraph with Edge Dependent Vertex Weights, the same methodology used in the LAI paper and was originally outlined in Hayashi et al. \cite{10.1145/3340531.3412034}.
The Adjusted Rand Index (ARI) is computed as in \cite{10.1145/3340531.3412034} to evaluate clustering accuracy.
Over ten runs with random initializations, the mean residual when clustering this dataset for SNMPBB was $\num{0.7365}$ with a mean ARI of $\num{0.2798}$; the mean residual for SymPGNCG was $\num{0.7357}$ with a mean ARI of $\num{0.2954}$.
% These values are mostly similar to \cite{HaEtAl25}.
% \edits{However, we can see that LAI-SNMPBB converges much quicker on average than LAI-PGNCG.
% This shows a benefit of such PG methods wherein the per-iteration time cost is often small.
% Similar to how it is used in other machine learning methods, gradient descent can find as good if not better solutions compared to more complicated algorithms. when different modifications and techniques are used.
% }

\begin{figure}[h]
\centering
\begin{tikzpicture}
\begin{axis}[
    xlabel={Time (sec)},
    ylabel={Residual},
    grid=none,
    width=10cm,
    height=6cm,
    no markers,
    xmin=10,
]
\addplot+[mark=*, color=black, ultra thick] table[x index=0, y index=1, col sep=comma] {data/wos_compare.csv};
\addplot+[mark=square*, color=gray, thick] table[x index=0, y index=2, col sep=comma] {data/wos_compare.csv};
\end{axis}
\end{tikzpicture}

\begin{tikzpicture}
\begin{axis}[
            hide axis,
            xmin=0, xmax=1,
            ymin=0, ymax=1,
            legend columns=5,
            legend style={
                at={(0.8,1.05)}, % 0.5
                anchor=south,
                draw=none,
                /tikz/every even column/.append style={column sep=0.5cm}
            },
            no markers
        ]
\addlegendimage{color=black,ultra thick}
\addlegendentry{LAI-SNMPBB (ours)}
\addlegendimage{color=gray, thick}
\addlegendentry{LAI-SymPGNCG \cite{10.1145/3340531.3412034}}
\end{axis}
\end{tikzpicture}
\caption{Residual $\|V - WH\|/\|V\|$ vs. wall-clock time on the Web of Science dataset (46985 term documents, 7 categories, rank $r = 7$, averaged over 10 runs).
Both algorithms use the Hypergraph-with-Edge-Dependent-Vertex-Weights preprocessing of Hayashi et al.\ \cite{10.1145/3340531.3412034}.
LAI-SNMPBB and LAI-SymPGNCG reach comparable final residuals ($7.365 \times 10^{-1}$ vs. $7.357 \times 10^{-1}$), but LAI-SNMPBB does so in consistently less time across all runs.}
\label{fig:wos}

\end{figure}
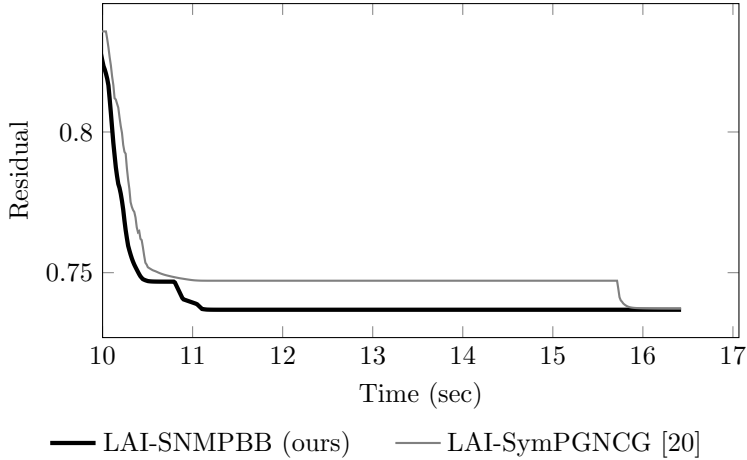

% \subsection{Overview}

\section{Conclusion}

This work demonstrates the effectiveness of projected gradient methods in symmetric nonnegative matrix factorization. In particular, we show that PGD-based alternating methods augmented with nonmonotone line search, Barzilai-Borwein step-sizes, and judicious regularization routinely exceed the performance of established approaches like SymANLS and SymNewton.

% , to our knowledge, 
The proposed SNMPBB algorithm is the first adaptation of nonmonotone projected Barzilai-Borwein methods to the symmetric NMF setting.
Like some other algorithms, rather than enforcing symmetry through a direct variable identification, the method maintains two coupled variables through a quadratic penalty formulation.
We further extend the approach with graph Laplacian regularization, resulting in the Graph-SNMPBB algorithm for clustering applications.

The numerical experiments demonstrate that these mechanisms substantially improve the practical behavior of projected gradient methods for symmetric problems.
% When factoring synthetic matrices, SNMPBB consistently achieved comparable or lower residuals than SymANLS and SymNewton while requiring significantly less computational time, particularly for higher-rank factorizations.
% For graph clustering tasks, Graph-SNMPBB matched or exceeded competing methods across several real-world datasets while converging faster on large and sparse problems.

We additionally incorporated the recent low-rank approximate input methodology into SNMPBB and show both theoretically and experimentally that the Barzilai–Borwein curvature information is preserved under the approximation.
The resulting LAI-SNMPBB method outperformed LAI-SymPGNCG on SuiteSparse matrices in both runtime and residual quality, with empirical evidence suggesting that limited inner iterations act as an implicit regularization mechanism.

Overall, the results suggest that projected gradient methods for symmetric NMF have been historically underestimated.
When combined with appropriate curvature scaling, globalization strategies, and regularization, first-order methods can provide an effective balance between computational efficiency, scalability, and solution quality for both matrix factorization and graph clustering problems.

\section*{Acknowledgments}
We are very grateful to Koby Hayashi for providing codes for the LAI algorithms and experiments.

\bibliographystyle{siamplain}
\bibliography{my}
\end{document}

%% file: shared.tex
\usepackage{lipsum}
\usepackage{amsfonts}
\usepackage{graphicx}
\usepackage{epstopdf}
\usepackage{algorithmic}

\usepackage{amsmath,amssymb}
\usepackage[scientific-notation=true]{siunitx}
\usepackage{subcaption}
\usepackage{changepage}
\usepackage{url}
\usepackage{algorithm}
\usepackage{pgfplotstable}
\usepackage{xcolor}
\usepackage{booktabs}
\usepackage{tikz}            % main package for creating pgf plots
\usepackage{pgfplots}
\usepackage{caption}
\pgfplotsset{compat=1.14} 
\usepackage[justification=centering]{subcaption}

\usepackage{hyperref} % [hidelinks]

\ifpdf
  \DeclareGraphicsExtensions{.eps,.pdf,.png,.jpg}
\else
  \DeclareGraphicsExtensions{.eps}
\fi

\usepackage{fixme}
\fxsetup{
  status=draft,
  theme=color,
  silent,
}
\fxsetup{marginface=\linespread{1}\footnotesize}
\FXRegisterAuthor{jb}{jbe}{\colorbox{red!20}{\color{blue}JB}}

\newsiamremark{remark}{Remark}
\newsiamremark{hypothesis}{Hypothesis}
\crefname{hypothesis}{Hypothesis}{Hypotheses}
\newsiamthm{claim}{Claim}
\newsiamremark{fact}{Fact}
\crefname{fact}{Fact}{Facts}

\headers{Symmetric NMF}{R. Swart and J. J. Brust}

\title{A Nonmonotone Gradient-Based Algorithm for Symmetric Nonnegative Matrix Factorization and Graph Clustering\thanks{Submitted to the editors Spring/Summer 2026.
\funding{This work was partially supported by the startup fund at Arizona State University Grant PG16270 and the Simons Travel Support for Mathematicians GR48329}}}

\author{Ryan Swart%
  \thanks{
    School of Mathematical and Statistical Sciences, 
    Arizona State University, Tempe, AZ 
    (\email{rswart@asu.edu}, \email{jjbrust@asu.edu}).
  }
\and Johannes J.~Brust\footnotemark[2]}

\usepackage{amsopn}